\documentclass[10pt,journal,compsoc]{IEEEtran}
\usepackage{times}
\usepackage{epsfig}
\usepackage{graphicx}
\usepackage{amsmath}
\usepackage{amssymb}
\usepackage{amsthm}
\usepackage{amsfonts}
\usepackage{comment}
\usepackage{subfigure}
\usepackage{graphicx}
\usepackage{mathrsfs}
\usepackage{amssymb}
\usepackage{bm}
\usepackage{array}
\usepackage{color}
\usepackage{multirow}

\usepackage{algorithm}
\usepackage{algorithmic}
\usepackage{url}

\ifCLASSOPTIONcompsoc
  \usepackage[nocompress]{cite}
\else
  \usepackage{cite}
\fi
\ifCLASSINFOpdf
\else
\fi
\begin{document}
%
% paper title
% Titles are generally capitalized except for words such as a, an, and, as,
% at, but, by, for, in, nor, of, on, or, the, to and up, which are usually
% not capitalized unless they are the first or last word of the title.
% Linebreaks \\ can be used within to get better formatting as desired.
% Do not put math or special symbols in the title.
\title{Clustering with Outlier Removal}
%
%
% author names and IEEE memberships
% note positions of commas and nonbreaking spaces ( ~ ) LaTeX will not break
% a structure at a ~ so this keeps an author's name from being broken across
% two lines.
% use \thanks{} to gain access to the first footnote area
% a separate \thanks must be used for each paragraph as LaTeX2e's \thanks
% was not built to handle multiple paragraphs
%
%
%\IEEEcompsocitemizethanks is a special \thanks that produces the bulleted
% lists the Computer Society journals use for "first footnote" author
% affiliations. Use \IEEEcompsocthanksitem which works much like \item
% for each affiliation group. When not in compsoc mode,
% \IEEEcompsocitemizethanks becomes like \thanks and
% \IEEEcompsocthanksitem becomes a line break with idention. This
% facilitates dual compilation, although admittedly the differences in the
% desired content of \author between the different types of papers makes a
% one-size-fits-all approach a daunting prospect. For instance, compsoc
% journal papers have the author affiliations above the "Manuscript
% received ..."  text while in non-compsoc journals this is reversed. Sigh.

\author{Hongfu Liu,~\IEEEmembership{Member,~IEEE,}
        Jun Li,~\IEEEmembership{Member,~IEEE,} Yue Wu
        and Yun Fu,~\IEEEmembership{Fellow,~IEEE}% <-this % stops a space
\IEEEcompsocitemizethanks{\IEEEcompsocthanksitem H. Liu is with Department of Computer Science, Brandeis University, Waltham; Email: hongfuliu@brandeis.edu. J. Li (co-first author) is with Laboratory for Computational Physiology at Massachusetts Institute of Technology; Email: junli@mit.edu. Y. Wu and Y. Fu are with Department of Electrical \& Computer Engineering, Northeastern University, Boston; Email: \{yuewu, yunfu\}@ece.neu.edu. \protect
% note need leading \protect in front of \\ to get a newline within \thanks as
% \\ is fragile and will error, could use \hfil\break instead.
}% <-this % stops an unwanted space
\thanks{Manuscript received XXX; revised XXX.}}

% note the % following the last \IEEEmembership and also \thanks -
% these prevent an unwanted space from occurring between the last author name
% and the end of the author line. i.e., if you had this:
%
% \author{....lastname \thanks{...} \thanks{...} }
%                     ^------------^------------^----Do not want these spaces!
%
% a space would be appended to the last name and could cause every name on that
% line to be shifted left slightly. This is one of those "LaTeX things". For
% instance, "\textbf{A} \textbf{B}" will typeset as "A B" not "AB". To get
% "AB" then you have to do: "\textbf{A}\textbf{B}"
% \thanks is no different in this regard, so shield the last } of each \thanks
% that ends a line with a % and do not let a space in before the next \thanks.
% Spaces after \IEEEmembership other than the last one are OK (and needed) as
% you are supposed to have spaces between the names. For what it is worth,
% this is a minor point as most people would not even notice if the said evil
% space somehow managed to creep in.

% The paper headers
\markboth{IEEE TRANSACTIONS ON KNOWLEDGE AND DATA ENGINEERING, VOL. XX, NO. XX, APRIL 2019}%
{Shell \MakeLowercase{\textit{et al.}}: Bare Demo of IEEEtran.cls for Computer Society Journals}
% The only time the second header will appear is for the odd numbered pages
% after the title page when using the twoside option.
%
% *** Note that you probably will NOT want to include the author's ***
% *** name in the headers of peer review papers.                   ***
% You can use \ifCLASSOPTIONpeerreview for conditional compilation here if
% you desire.

% The publisher's ID mark at the bottom of the page is less important with
% Computer Society journal papers as those publications place the marks
% outside of the main text columns and, therefore, unlike regular IEEE
% journals, the available text space is not reduced by their presence.
% If you want to put a publisher's ID mark on the page you can do it like
% this:
%\IEEEpubid{0000--0000/00\$00.00~\copyright~2015 IEEE}
% or like this to get the Computer Society new two part style.
%\IEEEpubid{\makebox[\columnwidth]{\hfill 0000--0000/00/\$00.00~\copyright~2015 IEEE}%
%\hspace{\columnsep}\makebox[\columnwidth]{Published by the IEEE Computer Society\hfill}}
% Remember, if you use this you must call \IEEEpubidadjcol in the second
% column for its text to clear the IEEEpubid mark (Computer Society jorunal
% papers don't need this extra clearance.)

% use for special paper notices
%\IEEEspecialpapernotice{(Invited Paper)}

% for Computer Society papers, we must declare the abstract and index terms
% PRIOR to the title within the \IEEEtitleabstractindextext IEEEtran
% command as these need to go into the title area created by \maketitle.
% As a general rule, do not put math, special symbols or citations
% in the abstract or keywords.
\IEEEtitleabstractindextext{%
\begin{abstract}
Cluster analysis and outlier detection are strongly coupled tasks in data mining area. Cluster structure can be easily destroyed by few outliers; on the contrary, outliers are defined by the concept of cluster, which are recognized as the points belonging to none of the clusters. Unfortunately, most existing studies do not notice the coupled relationship between these two task and handle them separately. In light of this, we consider the joint cluster analysis and outlier detection problem, and propose the Clustering with Outlier Removal (COR) algorithm. Generally speaking, the original space is transformed into the binary space via generating basic partitions in order to define clusters. Then an objective function based Holoentropy is designed to enhance the compactness of each cluster with a few outliers removed. With further analyses on the objective function, only partial of the problem can be handled by K-means optimization. To provide an integrated solution, an auxiliary binary matrix is nontrivally introduced so that COR completely and efficiently solves the challenging problem via a unified K-means{-}{-} with theoretical supports. Extensive experimental results on numerous data sets in various domains demonstrate the effectiveness and efficiency of COR significantly over state-of-the-art methods in terms of cluster validity and outlier detection. Some key factors in COR are further analyzed for practical use. Finally, an application on flight trajectory is provided to demonstrate the effectiveness of COR in the real-world scenario.
\end{abstract}

% Note that keywords are not normally used for peerreview papers.
}

\renewcommand{\algorithmicrequire}{\textbf{Input:}}  %Use Input in the format of Algorithm
\renewcommand{\algorithmicensure}{\textbf{Output:}}  %Use Output in the format of Algorithm
\newtheorem{lemma}{Lemma}%[section]
\newtheorem{remark}{Remark}%[section]
\newtheorem{theorem}{Theorem}%[section]
\newtheorem{definition}{Definition}%[section]

% make the title area
\maketitle

\section{Introduction}
Cluster analysis is a fundamental task in data mining and machine learning area, which aims to separate a bunch of data points into different groups so that similar points are assigned into the same cluster. Although cluster analysis has been studied for long time, it is still catching rising attention in industrial scenarios due to its wide applications, from customer segmentation~\cite{tsai2015customer} to information retrieval~\cite{campos2015survey}, and from recommendation systems~\cite{shepitsen2008personalized} to resource allocation~\cite{grandl2015multi}. Accordingly, cluster analysis has also been extensively explored in the academia. K-means is one of the most representative clustering methods, which seeks $K$ prototypes as the centroids to present the data points with the nearest distance. Spectral clustering designed for graph partition, minimizes the weights of cut edges to obtain disconnected sub-graphs with roughly even sizes. Gaussian mixture model estimates $K$ Gaussian distribution with means and variances to fit the data.

%Beyond employing the raw features for clustering, some methods aim to learn the more effective representation with some constraints. Low-rank representation assumes that the intrinsic or clean data lie in low-rank manifold; subspace sparse clustering explores self-expression property with sparse coefficient for representation learning. Recently, consensus clustering generates basic partition first, and employs the basic partitions as the representation for robust partition.

Although tremendous efforts have been devoted in the cluster analysis, most of the existing methods assume that all the data points should be assigned a cluster label. In another word, there are no anomaly data points for during clustering process. Unfortunately, this is not always true, especially for the unsupervised task. The potential anomalies or outliers inevitably degrade the clustering performance. For example, few outliers easily destroy the cluster structure derived from K-means and generate bizarre distributions of Gaussian mixture model. To handle outliers or noisy data, some robust clustering methods have been proposed to recover the clean data. Metric learning aims to learn a robust distance function to resist the outliers~\cite{Davis07ICML}; $L_{1}$ norm is employed to alleviate the negative impact of outliers on the cluster structure~\cite{Ding06ICML}. Beyond these, some methods aim to learn the more effective representation with some constraints. Low-rank representation assumes that the intrinsic or clean data lie in low-dimensional manifold~\cite{Liu10ICML}; subspace sparse clustering explores self-expression property with sparse coefficient for representation learning. Recently, consensus clustering generates basic partition first, and employs the basic partitions as the representation for robust partition. Note that these methods still assign the cluster labels for each data point, rather than explicitly removing anomaly points.

To tackle the negative impacts of outliers during the clustering processing, some unsupervised outlier detection methods have been put forward from different aspects. Usually each data point is calculated a score to identify the outlier degree, returning top $K$ outlier candidates. Local outlier factor is one of the popular density-based methods, where outliers are identified by comparing the local density of the data point and its neighbors\cite{Breunig00SIR}. Similarly, local distance-based outlier detection uses the relative location of an object to its neighbours to determine the degree to which the object deviates from its neighbourhood~\cite{Zhang09PKDD}. Angle-based outlier detection focuses on variance in the angles between the difference vectors of a point to the other points, where the angles of the outliers and other two randomly selected points have some deviations\cite{Kriegel08KDD,Pham12KDD}. Other representative methods include ensemble-based iForest~\cite{Liu08ICDM}, eigenvector-based OPCA~\cite{Lee13TKDE}, cluster-based TONMF~\cite{Kannan17SDM}, and so on.

Although outlier detection methods can be regarded as a pre-process for cluster analysis, outlier detection and cluster analysis are usually conducted as two separated tasks. In fact, they are strongly coupled. Cluster structure can be easily destroyed by few outliers~\cite{Georgogiannis16NIPS}; on the contrary, outliers are defined by the concept of cluster, which are recognized as the points belonging to none of the \mbox{clusters}~\cite{Breunig00SIR}. However, few of the existing works treat the cluster analysis and outlier detection in a unified framework. DBSCAN is one of the pioneering works for density-based cluster analysis with the outlier set as an extra output~\cite{ester1996density}, where all the data points are divided in three categories, core points, border points and outliers according to the density, then the clusters are generated by connecting core points and their affiliated border points. Strictly DBSCAN does not belong to the joint cluster analysis and outlier detection, which identifies and removes the outliers first and then follows the cluster generation. To our best knowledge, K-means{-}{-}~\cite{Chawla13SDM} is the first work along this direction. It aims to detect $o$ outliers and partition the rest points into $K$ clusters, where the instances far away from the nearest centroid are regarded as outliers during clustering process. Since this problem is a discrete optimization problem in essence, it is natural that Langrangian Relaxation (LP)~\cite{Ott14NIPS} formulates the clustering with outliers as an integer programming problem with several constraints, which requires the cluster creation costs as the input parameter. Although these two pioneering works provide new directions for joint clustering and outlier detection, the spherical structure assumption of K-means{-}{-} and the original feature space limit its capacity for complex data analysis, and the setup of input parameters and high time complexity in LP make it infeasible for large-scale data.

In this paper, we focus on the joint cluster analysis and outlier detection problem, and propose the Clustering with Outlier Removal (COR) algorithm. Since the outliers are relied on the concept of clusters, we transform the original space into the partition space via running some clustering algorithms (e.g. K-means) with different parameters to generate a set of different basic partitions. By this means, the continuous data are mapped into a binary space via one hot encoding of basic partitions. In the partition space, an objective function is designed based on Holoentropy~\cite{Wu13TKDE} to increase the compactness of each cluster after some outliers are removed. With further analyses, we transform the partial problem of the objective function into a K-means optimization. To provide a complete and neat solution, an auxiliary binary matrix derived from basic partitions is introduced. Then COR is conducted on the concatenated matrix, which completely and efficiently solves the challenging problem via a unified K-means{-}{-} with theoretical supports. To evaluate the performance of COR, we conduct extensive experiments on numerous data sets in various domains. Compared with K-means{-}{-} and numerous outlier detection methods, COR outperforms rivals over in terms of cluster validity and outlier detection by four metrics. Moreover, we demonstrate the high efficiency of COR, which indicates it is suitable for large-scale and high-dimensional data analysis. Some key factors in COR are further analyzed for practical use. Finally, an application on flight trajectory is provided to demonstrate the effectiveness of COR in the real-world scenario. Here we summarize our major contributions as follows.

\begin{itemize}
  \item To our best knowledge, we are the first to conduct the clustering with outlier removal in the partition space, which achieves simultaneous consensus clustering and outlier detection.
  \item Based on Holoentropy, we design the objective function from the aspect of outlier detection, which is partially solved by K-means clustering. By introducing an auxiliary binary matrix, we completely transform the non K-means clustering problem into a K-means{-}{-} \mbox{optimization} with theoretical supports.
  \item Extensive experimental results demonstrated the effectiveness and efficiency of our proposed COR \mbox{significantly} over the state-of-the-art rivals in terms of cluster \mbox{validity} and outlier detection.
\end{itemize}

The rest of this paper is organized as follows. Section 2 introduces the related work on robust clustering, outlier detection and joint learning. Section 3 provides the preliminary knowledge and our problem formulation. In Section 4, we elaborate the equivalent relationship between our addressed problem and K-means{-}{-} with an augmented matrix. Section 5 delivers a thorough discussion on the relationship among COR and cluster analysis, outlier detection and consensus clustering. Extensive experiments are conducted in Section 6. Finally, we conclude this paper in Section 7.

\section{Related Work}
In this section, we present the related work in terms of robust clustering, consensus clustering, outlier detection, and highlight the difference between existing work and ours.

\subsection{Robust Clustering}
To alleviate the impact of outliers, robust clustering\footnote{The concept of robust clustering means that the partition is robust to outliers, rather than noisy features.} has been proposed from different aspects. From the distance function aspect, metric learning is used to learn a robust metric to measure the similarity between two points by taking the outliers into account~\cite{Davis07ICML,Yi12NIPS}; $L_{1}$ norm models the outliers as the sparse constraint for cluster analysis~\cite{Ding06ICML,Elhamifar13TPAMI}. From the data aspect, the outliers are assigned few weights during clustering process~\cite{Dotto16SC}; low-rank representation treats the data as the clean part and outliers, and constrains the clean part with the lowest rank~\cite{Liu10ICML}. From the model fusion aspect, ensemble clustering integrates different partitions into a consensus one to deliver a robust result~\cite{Strehl02JMLR,Liu17TKDE}. Although these robust clustering methods reduce the negative impacts of outliers on the cluster structure, they fail to explicitly detect or remove outlier points for clustering. In another word, each data point is assigned with a cluster label, even for the outliers.

\subsection{Consensus Clustering}
Consensus clustering, also known as ensemble clustering, targets to integrate several diverse partition results from traditional clustering methods into a consensus one~\cite{Strehl02JMLR}. It has been widely recognized of robustness, consistency, novelty and stability over traditional clustering methods, especially in generating robust partitions, discovering novel structures, handling noisy features, and integrating solutions from multiple sources. The process of consensus clustering generally has two steps: basic partitions generation and consensus fusion. Given basic partitions as input, consensus clustering is in essence a fusion problem rather than a partitioning problem, which seeks for an optimal combinatorial result from basic partitions. Over the past years, many clustering ensemble techniques have been proposed, resulting in various of ways to face the problem together with new fields of application for these techniques. Generally speaking, consensus clustering can be divided into two categories, \emph{i.e.}, those with or without an explicit global objective function. The methods that do not set objective functions make use of some heuristics or meta-heuristics to find approximate solutions. Representative methods include co-association matrix-based~\cite{Fred05TPAMI,Lourenco13ML}, graph-based~\cite{Strehl02JMLR,Fern04ICML}, relabeling and voting based~\cite{Ayad08TPAMI} and locally adaptive cluster-based algorithms~\cite{Domeniconi09TKDD}. On another hand, the methods with explicit objectives employ global objective functions to measure the similarity between basic partitions and the consensus one. Representative solutions include K-means-like algorithm~\cite{Topchy03ICDM}, NMF~\cite{Li07ICDM}, EM algorithm~\cite{Topchy04SDM}, simulated annealing~\cite{Lu08AAAI} and combination regularization~\cite{Xie14KDD}. More information on consensus clustering can be found in the recent survey~\cite{liu2019consensus}.

\subsection{Outlier Detection}
Outlier detection, also known as anomaly detection, seeks the points deviation from others and identifies these points as outliers, where most of the existing studies focus on unsupervised outlier detection. Some criteria are designed to assign a score to each point, and the points with large scores are regarded as the outlier candidates. Some representative methods include density-based LOF\cite{Breunig00SIR}, COF\cite{Tang02PKDD}, distance-based LODF~\cite{Zhang09PKDD}, frequent pattern-based Fp-outlier~\cite{He08CSIS}, angle-based ABOD~\cite{Kriegel08KDD} and its fast version FABOD~\cite{Pham12KDD}, ensemble-based iForest~\cite{Liu08ICDM}, BSOD~\cite{Liu16ICBD}, eigenvector-based OPCA~\cite{Lee13TKDE}, cluster-based TONMF~\cite{Kannan17SDM}. Recently, there are deep learning based outlier detection methods such as deep one-class SVM~\cite{ruff2018deep} and GAN-based methods~\cite{schlegl2017unsupervised,zenati2018efficient,li2018anomaly}, which learns a non-linear transformation to project the original data into hidden space for effective recognition.
However, these methods train the model only with clear data, and predict new data whether they are outliers, which is different from the problem we address here.

\subsection{Joint Clustering and Outlier Detection}
Cluster analysis and outlier detection are consistently hot topics in data mining area; however, they are usually considered as two independent tasks. Although robust clustering resists to the impact of outliers, each point including outliers is assigned the cluster label. Few of the existing works treat the cluster analysis and outlier detection in a unified framework. Two-stage frameworks, such as DBSCAN conduct the outlier detection first, then apply the clustering method for partition, which becomes struggled to handle complex data. K-measn{-}{-}~\cite{Chawla13SDM} detects $o$ outliers and partitions the rest points into $K$ clusters, where the instances with large distance to the nearest centroid are regarded as outliers during the clustering process. Langrangian Relaxation (LP)~\cite{Ott14NIPS} formulates the clustering with outliers as an integer programming problem, which requires the cluster creation costs as the input parameter. This problem has also been theoretically studied in facility location. Charikar \textit{et al.} proposed a bi-criteria approximation algorithm for the facility location with outliers problem~\cite{Charikar01SODA}. Chen proposed a constant factor approximation algorithm for the K-median with outliers problem~\cite{Chen08SODA}.

In this paper, we consider the clustering with outlier removal problem, which partitions the entire data sets into several clusters and one outlier set. Although some pioneering works provide new directions for joint clustering and outlier detection, none of these algorithms expect K-means{-}{-} are amenable to a practical implementation on large data sets, while of theoretical interests. Moreover, the spherical structure assumption of K-means{-}{-} and the original feature space limit its capacity for complex data analysis. In light of this, we transform the original feature space into the partition space, where based on Holoentropy, the COR is designed to achieve simultaneous consensus clustering and outlier detection.

\section{Problem Formulation}
In this section, we first illustrate some preliminary knowledge and elaborate our objective function for clustering and outlier removal.

\subsection{Preliminaries}
Here we introduce some basic knowledge on K-means{-}{-} and Holoentropy.

K-means{-}{-}~\cite{Chawla13SDM} is a variant of K-means, which is particularly designed for handling the sensitivity of K-means on outliers. It is widely recognized that few outliers deviate the centroids from their intrinsic positions. To tackle with this, some data points with far distance to their centroids are regarded as the outlier candidates, which are not assigned with any cluster label and involved into the centroid updating, either. Similar to K-means, K-means{-}{-} also has two iterative stages, data point assignment and centroid updating. During the data point assignment, we calculate the distances between each data point and its nearest centroid, and sort the distances, where the data points with top $o$ largest distances are outlier candidates. For the centroid updating, it is the same with K-means since these outlier candidates are not assigned with cluster labels. It is worthy to note that the outlier candidates are changing during the iteration. Compared with K-means, K-means{-}{-} requires two input parameters, the numbers of clusters and outlier $K$ and $o$. It enjoys many properties as K-means in terms of neat mathematical formulation, model efficiency and convergence.

As pointed out by Ref~\cite{Wu13TKDE}, it is not suitable to only employ entropy or total correlation for outlier detection. They proposes a new measure Holoentropy as follows.

\begin{definition}[Holoentropy]\label{def:holo} Holoentropy $HL(\mathcal{Y})$ is defined as the sum of the entropy and the total correlation of the random vector $\mathcal{Y}$, and can be expressed by the sum of the entropies on all attributes.
\end{definition}

Holoentropy is an outlier detection metric based on information theory, which handles the categorical data and takes both entropy and total correlation into consideration. In the rest of this paper, we elaborate our proposed objective function based on Holoentropy, and derive its to K-means{-}{-} algorithm for a neat and efficient solution.

\subsection{Objective Function}
Cluster analysis and outlier detection are closely coupled tasks. Cluster structure can be easily destroyed by few outlier points; on the contrary, outliers are defined by the concept of cluster, which are recognized as the points belonging to none of the clusters. To cope with this challenge, we focus on the Clustering with Outlier Removal (COR). Specifically, the outlier detection and clustering tasks are jointly conducted, where $o$ points are detected as the outliers and the rest instances are partitioned into $K$ clusters. Table~\ref{tab:notation} shows the notations used in the following sections.

The cluster structure is vulnerable to few outliers, and outliers request to be identified with cluster boundary. The coupling relationship among cluster analysis and outlier detection makes it like a chicken-and-egg problem. To escape the chicken-and-egg problem in joint clustering and outlier detection, we are inspired by consensus clustering~\cite{Strehl02JMLR,Fred05TPAMI,Domeniconi09TKDD}, which incorporates several basic partitions generated from the data for a robust fusion to alleviate the negative effects from outliers. Moreover, the definition of outliers relies on the clusters. The above two points motivate us to transform the data from the original feature space into partition space via generating several basic partitions. This process is similar to generate basic partitions in consensus clustering~\cite{Liu15KDD,Liu16KDD}. Let $X$ denote the data matrix with $n$ points and $d$ features. A partition of $X$ into $K$ crisp clusters can be represented as a collection of $K$ subsets of objects with a label vector $\pi=(L_\pi(x_1),\cdots,L_\pi(x_n)),1\le l\le n$, where $L_\pi(x_l)$ maps $x_l$ to one of the $K$ labels in $\{1,2,\cdots,K\}$. Some basic partition generation strategy, such as K-means clustering with different cluster numbers can be applied to obtain $r$ basic partitions $\Pi=\{\pi_i\},1\le i \le r$. Let $K_i$ denote the cluster number for $\pi_i$ and $R = \sum_{i=1}^rK_i$. Then a binary matrix $B=\{b_l\},1\le l\le n$ can be derived from $\Pi$ as follows:
\begin{equation}\label{eq:binary}
\begin{split}
b_l&=(
b_{l,1},\cdots,b_{l,i},\cdots,b_{l,r}),~\textrm{with} \\
b_{l,i}&=( b_{l,i1},\cdots,b_{l,ij},\cdots,b_{l,iKi}),~\textrm{and} \\
b_{l,ij}&=\left\{
  \begin{array}{ll}
    1,&\textrm{if}~L_{\pi_i}(x_l)=j\\
    0,&\textrm{otherwise}
  \end{array}
\right..
\end{split}
\end{equation}

It is worthy to note that we do not require a specific algorithm to generate basic partitions. For the sake of simplicity and efficiency, K-means with different cluster numbers are recommended to generated basic partitions. Although K-means is vulnerable to outliers, our COR still delivers promising results based on the basic partitions generated by K-means. The benefits to transform the original space into the partition space lie in (1) the binary value indicates the cluster-belonging information, which is particularly designed according to the definition of outliers, and (2) compared with the continuous space, the binary space is much easier to identify the outliers due to the categorical features. For example, Holoentropy is a widely used outlier detection metric for categorical data~\cite{Wu13TKDE}.

\begin{table}[t!]
  \caption{The Contingency Matrix}\vspace{-0.2cm}
  \small
  \centering
  \begin{tabular}{ccl}
  \hline
    Notation & Domain  & Description\\
  \hline
    $n$ & $\mathcal{Z}$ & Number of instances\\
    $d$ & $\mathcal{Z}$ & Number of features\\
    $K$ & $\mathcal{Z}$ & Number of clusters\\
    $o$ & $\mathcal{Z}$ & Number of outliers\\
    $r$ & $\mathcal{Z}$ & Number of basic partitions\\
    $X$ & $\mathcal{R}^{\{n\times d\}}$ & Data set\\
    $O$ & $\mathcal{R}^{\{o\times d\}}$ & Outlier set\\
    $\Pi$ & $\mathcal{Z}^{\{n\times r\}}$& Set of basic partitions\\
    $B$ & $\{0,1\}^{n\times R}$ & Binary matrix derived from $\Pi$\\
  \hline
  \end{tabular} \label{tab:notation}\vspace{-2mm}
\end{table}

In Ref~\cite{Wu13TKDE}, the authors aimed to minimize the Holoentropy of the data set with $o$ outliers removed. Here we assume there exists the cluster structure within the whole data set. Therefore, it is more reasonable to minimize the Holoentropy of each cluster. In such a way, the clusters become compact after the outliers are removed, rather than the entire data set. Therefore, based on Holoentropy of each cluster, we give our objective function of COR as follows.
\begin{equation}\label{eq:obj-outlier}
  \min_{\pi} \sum_{k=1}^K p_{k}HL(C_k),
\end{equation}
where $HL(\cdot)$ is defined in Definition 1, $\pi$ is the cluster indicator, including $K$ clusters $C_1 \cup \cdots \cup C_{K} = X\backslash O$, with $C_k \cap C_{k'} = \emptyset$ if $k\neq k'$ and $p_{k+} = |C_{k}|/(n-o)$. Actually, the objective function in Eq.~\eqref{eq:obj-outlier} is the summation of weighted Holoentropy of each cluster, where the weight $p_{k}$ is proportional to the cluster size. Here the number of cluster $K$ and the number of outliers $o$ are two parameters of our proposed algorithm, which is the same setting with K-means{-}{-}~\cite{Chawla13SDM}, and we treat determining $K$ and $o$ as an orthogonal problem beyond this paper. In the next section, we provide an efficient solution for COR by introducing another auxiliary binary matrix.

\section{Clustering with Outlier Removal}
To solve the problem in Eq.~\eqref{eq:obj-outlier}, we provide a detailed objective function on the binary matrix $B$ as follows.
\begin{equation}\label{eq:obj-outlier2}
\begin{split}
  &\sum_{k=1}^K p_{k}HL(C_k) \propto \sum_{k=1}^K \sum_{b_{l} \in C_k}\sum_{i=1}^r\sum_{j=1}^{K_i} H(C_{k,ij}),\textrm{and} \\
  &H(C_{k,ij})=-(1-p_{k,ij})\log (1-p_{k,ij})-p_{k,ij}\log p_{k,ij},\\
\end{split}
\end{equation}
where $H$ denotes the Shannon entropy and $p_{k,ij}$ denotes the probability of $b_{l,ij}=1$ in the $ij$-th column of $C_k$.

To better understand the meaning of $p_{k,ij}$ in Eq.~\eqref{eq:obj-outlier2}, we provide the following lemma.
\begin{lemma}\label{lem:mk}
For K-means clustering on the binary data set $B$, the $k$-th centroid satisfies
\begin{equation}\label{eq:mk}
\begin{split}
m_{k} &=(m_{k,1},\cdots,m_{k,i},\cdots,m_{k,r}), ~\text{with}\\
m_{k,i} &=(m_{k,i1}, \cdots m_{k,ij} \cdots m_{k,iK_i}),~\textrm{and}\\
m_{k,ij} &= \sum_{b_{l,ij} \in C_k} b_{l,ij}/|C_k| = p_{k,ij}, \forall~k,i,j.\\
\end{split}
\end{equation}
\end{lemma}

The proof of Lemma~\ref{lem:mk} is self-evident according to the arithmetic mean of the centroid in K-means clustering. Based on Lemma~\ref{lem:mk}, we uncover the bridge between the problem in Eq.~\eqref{eq:obj-outlier2} and K-means clustering on the binary matrix $B$.

\begin{theorem}\label{the:second}
If K-means is conducted on $n-o$ inliers of the binary matrix $B$, we have
\begin{equation}\label{eq:equal}
\begin{split}
&\max \sum_{k=1}^K p_k\sum_{i=1}^r\sum_{j=1}^{K_i} p_{k,ij}\log p_{k,ij} \Leftrightarrow \min \sum_{k=1}^K \sum_{b_l \in C_k} f(b_l,m_{k}),\\
\end{split}
\end{equation}
where $m_{k}$ is the $k$-th centroid by Eq.~\eqref{eq:mk} and the distance function $f(b_l,m_{k})=\sum_{i=1}^r\sum_{j=1}^{K_i} D_{\textrm{KL}}(b_{l,ij}||m_{k,ij})$, here $D_{\textrm{KL}}(\cdot || \cdot)$ is the KL-divergence.
\end{theorem}

\begin{proof}
According to the Bregman divergence~\cite{Banerjee05JMLR}, we have $D_{\textrm{KL}}(s||t) = H(t) - H(s) + (s-t)^\top\nabla H(t)$, where $s$ and $t$ are two vectors with the same length. Then we start on the right side of Eq.~\eqref{eq:equal}.
\begin{equation}
\begin{split}
&\sum_{k=1}^K \sum_{b_l \in C_k} f(b_l,m_{k})\\
=&\sum_{k=1}^K \sum_{b_l \in C_k}\sum_{i=1}^r\sum_{j=1}^{K_i} (H(m_{k,ij})-H(b_{l,ij})\\
&+(b_{l,ij}-m_{k,ij})^\top\nabla H(m_{k,ij}))\\
=&\sum_{k=1}^K |C_k|\sum_{i=1}^r\sum_{j=1}^{K_i}H(m_{k,ij}) -\sum_{k=1}^K \sum_{b_l \in C_k}\sum_{i=1}^r\sum_{j=1}^{K_i}H(b_{l,ij}).\\
\end{split}
\end{equation}
The above equation holds due to $\sum_{b_l \in C_k}(b_{l,ij}-m_{k,ij})=0$, and the second term is a constant given the binary matrix $B$. According to Lemma~\ref{lem:mk}, we finish the proof.
\end{proof}

\begin{remark}
Theorem~\ref{the:second} uncovers the equivalent relationship between the second part in Eq.~\eqref{eq:obj-outlier2} and K-means on the binary matrix $B$. By this means, some part of this complex problem can be efficiently solved by the simple K-means clustering with KL-divergence on each dimension.
\end{remark}

Although Theorem~\ref{the:second} formulates the second part in Eq.~\eqref{eq:obj-outlier2} into a K-means optimization problem on the binary matrix $B$, there still remains two challenges. (1) The first part in Eq.~\eqref{eq:obj-outlier2} is difficult to formulate into a K-means objective function, and (2) Lemma~\ref{lem:mk} and Theorem~\ref{the:second} are conducted on $n-o$ inliers, rather than the whole matrix $B$. In the following, we focus on these two challenges, respectively.

The second part in Eq.~\eqref{eq:obj-outlier2} can be solved by K-means clustering, which inspires us to make efforts in order to transform the complete problem into the K-means solution. Since $1-p_{k,ij}$ is difficult involved into the K-means clustering by Theorem~\ref{the:second}, which means $1-p_{k,ij}$ cannot be modeled by the binary matrix $B$, here we aim to model it by introducing another binary matrix $\widetilde{B}=\{\widetilde{b}_l\}, 1\leq l\leq n$ as follows.
\begin{equation}\label{eq:binary2}
\begin{split}
\widetilde{b}_l&=(
\widetilde{b}_{l,1},\cdots,\widetilde{b}_{l,i},\cdots,\widetilde{b}_{l,r}),~\textrm{with} \\
\widetilde{b}_{l,i}&=(\widetilde{b}_{l,i1},\cdots,\widetilde{b}_{l,ij},\cdots,\widetilde{b}_{l,iKi}),~\textrm{and} \\
\widetilde{b}_{l,ij}&=\left\{
  \begin{array}{ll}
    0,&\textrm{if}~L_{\pi_i}(x_l)=j\\
    1,&\textrm{otherwise}
  \end{array}
\right..
\end{split}
\end{equation}

From Eq.~\eqref{eq:binary2}, $\widetilde{B}$ is also derived from $\Pi$. Compared with the binary matrix $B$ in Eq.~\eqref{eq:binary}, $\widetilde{B}$ can be regarded as the flip of $B$. In fact, $B$ and $\widetilde{B}$ are the 1-of-$K_i$ and ($K_i$-1)-of-$K_i$ codings of the original data, respectively. Based on $\widetilde{B}$, we can define $\widetilde{m}_{k}$ according to Eq.~\eqref{eq:mk}, then we have $\widetilde{m}_{k,ij}=1-{m}_{k,ij} = 1-p_{k,ij}$.

Based on the binary matrices $B$ and $\widetilde{B}$, we transform the problem in Eq.~\eqref{eq:obj-outlier2} into a unified K-means optimization by the following theorem.

\begin{algorithm}[t!]
\caption{Clustering with Outlier Removal} \label{alg}
\begin{algorithmic}[1]
\REQUIRE
$X$: data matrix; \\
\ \ \ \ \ \ $K, o, r$: number of clusters, outliers, basic partitions.\\
\ENSURE
$K$ clusters $C_1, \cdots C_K$ and outlier set $O$;
\STATE Generate $r$ basic partitions from $X$;
\STATE Build the binary matrices $B$ and $\widetilde{B}$ by Eq.~\eqref{eq:binary}\&\eqref{eq:binary2};
\STATE Initialize $K$ centroids from $[B\ \widetilde{B}]$;
\REPEAT
\STATE Calculate the distance between each point in $[B\ \widetilde{B}]$ and its nearest centroid;
\STATE Identify $o$ points with largest distance as outliers;
\STATE Assign the rest $n-o$ points to their nearest centroids;
\STATE Update the centroids by arithmetic mean;
\UNTIL the objective value in Eq.~\eqref{eq:obj-outlier} remains unchanged.
\end{algorithmic}
\end{algorithm}

\begin{theorem}\label{the:one}
If K-means is conducted on $n-o$ inliers of the binary matrix $[B\ \widetilde{B}]$, we have
\begin{equation}\nonumber
\min_{\pi} \sum_{k=1}^K p_{k}HL(C_k)\Leftrightarrow \min \sum_{k=1}^K \sum_{b_l \in C_k} (f(b_l,m_{k}) + f(\widetilde{b}_l,\widetilde{m}_{k})),
\end{equation}
where $m_{k}$, $\widetilde{m}_{k}$ are the $k$-th centroid by Eq.~\eqref{eq:mk}, and the distance function $f(b_{l},m_{k})=\sum_{i=1}^r\sum_{j=1}^{K_i} D_{\textrm{KL}}(b_{l,ij}||m_{k,ij})$, $f(\widetilde{b}_{l},\widetilde{m}_{k})=\sum_{i=1}^r\sum_{j=1}^{K_i} D_{\textrm{KL}}(\widetilde{b}_{l,ij}||\widetilde{m}_{k,ij})$, and $D_{\textrm{KL}}(\cdot || \cdot)$ is the KL-divergence.
\end{theorem}

\begin{remark}
The problem in Eq.~\eqref{eq:obj-outlier2} cannot be solved via K-means on the binary matrix $B$. Nontrivially, we introduce the auxiliary binary matrix $\widetilde{B}$, a flip of $B$, in order to model $1-p_{k,ij}$. By this means, the complete problem can be formulated by K-means clustering on the concatenated binary matrix $[B\ \widetilde{B}]$ in Theorem~\ref{the:one}. The benefits not only lie in simplifying the problem with a neat mathematical formulation, but also inherit the efficiency from K-means, which is suitable for large-scale data clustering with outlier removal.
\end{remark}

Theorem~\ref{the:one} completely solves the first challenge that the problem in Eq.~\eqref{eq:obj-outlier} with inliers with the auxiliary matrix $\widetilde{B}$. This makes a partial K-means solution into a complete K-means solution. In the following, we handle the second challenge, which conducts on the entire data points, rather than $n-o$ inliers.

In this paper, we consider the clustering with outlier removal, which simultaneously partitions the data and discovers outliers. That means the outlier detection and clustering are conducted in a unified framework. Since the centroids in K-means clustering are vulnerable to outliers, these outliers should not contribute to the centroids. Inspired by \mbox{K-means{-}{-}}~\cite{Chawla13SDM}, the outliers are identified as the points with large distance to the nearest centroid. The major difference is that K-means{-}{-} is proposed on the original feature space, while our problem starts from the Holoentropy on the partition space, and we formulate the problem into a K-means optimization with the auxiliary matrix $\widetilde{B}$. After delicate transformation and derivation, \mbox{K-means{-}{-}} is used as a tool to solve the problem in Eq.~\eqref{eq:obj-outlier}, which returns $K$ clusters $C_1, \cdots, C_K$ and outlier set $O$. The complete process of our proposed clustering with outlier removal is summarized in Algorithm~\ref{alg}.

%Thanks to Theorem~\ref{the:one}, we formulate the problem in Eq.~\eqref{eq:obj-outlier} with inliers into K-means framework so that the second challenge can fortunately solved by K-means{-}{-} on $[B\ \widetilde{B}]$, where we calculate the distance between each point and its corresponding nearest centroid, and label $o$ points as outliers with the largest distance. In light of this, we propose our clustering with outlier removal in \mbox{Algorithm}~\ref{alg}. The complex clustering with outlier removal problem in \mbox{Eq.~\eqref{eq:obj-outlier}} can be exactly solved by the existing \mbox{K-means{-}{-}} algorithm on the binary concatenated matrix $[B\ \widetilde{B}]$. The major difference is that K-means{-}{-} is proposed on the original feature space, while our problem starts from the Holoentropy on the partition space, and we formulate the problem into a K-means optimization with the auxiliary matrix $\widetilde{B}$. After delicate transformation and derivation, \mbox{K-means{-}{-}} is used as a tool to solve the problem in Eq.~\eqref{eq:obj-outlier}, which returns $K$ clusters $C_1, \cdots, C_K$ and outlier set $O$.

Next, we analyze the property of Algorithm~\ref{alg} in terms of time complexity and convergence. In Line-1, we first generate $r$ basic partitions, which are usually finished by K-means clustering with different cluster numbers. This step takes $\mathcal{O}(rt'\overline{K}nd)$, where $t'$ and $\overline{K}$ are the average iteration number and cluster number, respectively. Line 5-8 denotes the standard K-means{-}{-} algorithm, which has the similar time complexity $\mathcal{O}(tKnR)$, where $R = \sum_{i=1}^rK_i$ is the dimension of the binary matrix $B$ and $\widetilde{B}$. It is worthy to note that only $R$ elements are non-zero in $[B\ \widetilde{B}]$. In Line 6, we find $o$ points with largest distances, rather than sorting $n$ points so that it can be achieved with $\mathcal{O}(n)$. It is worthy to note that $r$ basic partitions can be generated via parallel computing, which dramatically decreases the execution time. Moreover, $t'$, $t$, $r$ and $R$ are relatively small compared with the number of points $n$. Therefore, the time complexity of our algorithm is roughly linear to the number of points, which easily scales up for big data clustering with outliers. Moreover, Algorithm~\ref{alg} is also guaranteed to converge to a local optimum by the following theorem.

\begin{theorem}\label{the:coverge}
Algorithm~\ref{alg} converges to a local optimum.
\end{theorem}

The proof holds due to the good convergence of K-means{-}{-}.

%\begin{proof}
%The classical K-means consists of two iterative steps, assigning data points to their nearest centroids and updating the centroids, which is guaranteed to converge to a local optimum with Bregman divergence~\cite{Banerjee05JMLR}. The distance function $f$ in our algorithm is the summation of KL-divergence on each dimension, which can be generalized by Bregman divergence. That means if we conduct K-means clustering on $n$ data points with $f$, the algorithm converges.

%Here K-means{-}{-} is utilized for the solution, which has the similar iterative steps. In the following, we analyze the objective function value change of these two steps. $n-o$ points are assigned labels during K-means{-}{-}, which means that $o$ outliers do not contribute to the objective function. Since the objective function value decreases in K-means with $n$ points assigned labels, the objective function value with $n-o$ points in Eq.~\eqref{eq:obj-outlier} decreases during the assignment phase in \mbox{K-means{-}{-}}. For the phase of updating the centroids,  arithmetic mean is optimal for the labeled $n-o$ points due to the fact that the derivation of the objective function to the centroids is zero. We finish the proof.
%\end{proof}

\section{Discussions}
In this section, we launch several discussions on clustering with outlier removal. Generally speaking, we elaborate it in terms of the traditional clustering, outlier detection and consensus clustering.

\textit{Comparison with cluster analysis.} Traditional cluster analysis aims to separate a bunch of points into different groups that the points in the same cluster are similar to each other. Each point is assigned with a hard or soft label. Although robust clustering is put forward to alleviate the impact of outliers, each point including outliers are assigned the cluster label. Differently, the problem we address here, clustering with outlier removal only assigns the labels for inliers and discovers the outlier set. Technically speaking, our COR belongs to the non-exhaustive clustering, where not all data points are assigned labels and some data points might belong to multiple clusters. NEO-K-Means~\cite{Whang15SDM} is one of the representative methods in this category. In fact, if we set the overlapping parameter to be zero in NEO-K-Means, it just degrades into K-means{-}{-}. Our COR is different from K-means{-}{-} in the feature space. The partition space not only naturally caters to the definition of outliers and Holoentropy, but also alleviates the spherical structure assumption of \mbox{K-means} optimization.

\textit{Comparison with outlier detection.} Outlier Detection is a hot research area, where tremendous efforts have been made to thrive this area from different aspects. Few of them simultaneously conduct cluster analysis and outlier detection. Except K-means{-}{-}, Langrangian Relaxation (LP)~\cite{Ott14NIPS} formulates the clustering with outliers as an integer programming problem, which requires the cluster creation costs as the input parameter. LP not only suffers from huge algorithmic complexity, but also struggles to set this parameter in practical scenarios, which leads LP to return the infeasible solutions. That is the reason that we fail to report the performance of LP in the experimental part. To our best knowledge, we are the first to solve the outlier detection in the partition space, and simultaneously achieve clustering and outlier removal. Our algorithm COR starts from the objective function in terms of outlier detection, and solves the problem via clustering tool, where demonstrates the deep connection between outlier detection domain and cluster analysis area.

\textit{Comparison with consensus clustering.} Consensus Clustering aims to fuse several basic partitions into an integrated one. In our previous work, we proposed K-means-based Consensus Clustering (KCC)~\cite{Wu13IJCAI,Wu15TKDE}, which transforms the complex consensus clustering problem into a K-means solution with flexible KCC utility functions. Similarly, the input of our COR is also a set of basic partitions, and it delivers the partition with outliers via K-means{-}{-}. The partition space derived from basic partitions enables COR not only to identify outliers, but also to fuse basic partition to achieve consensus clustering. From this view, Holoentropy can be regarded as the utility function to measure the similarity between the basic partition in $B$ or $\widetilde{B}$ and the final one. For the centroid updating, the missing values in basic partitions within KCC framework provide no utility, further do not contribute the centroids. For COR, we can automatically identify the outliers, which  do not participate into the centroid updating either.

\section{Experimental Results}
In this section, we first introduce the experimental settings and data sets , then showcase the effectiveness of our proposed method compared with K-means and K-means{-}{-}. Moreover, a variety of outlier detection methods are involved as the competitive methods. Some key factors in COR are further analyzed for practical use. Finally, an application on flight trajectory is provided to demonstrate the effectiveness of COR in the real-world scenario.

\subsection{Experimental Settings}
\textbf{Data sets.} To fully evaluate our COR algorithm, numerous data sets in different domains are employed. They include the gene expression data, image data, high-dimensional text data and other multivariate data. These data sets can be found from~\cite{Liu15SDM,Liu17DMKD} and UCI\footnote{\url{https://archive.ics.uci.edu/ml/datasets.html}}. Here we treat the class with smallest size as outliers. For \textit{ecoil}, three smallest classes in the original datasets are regarded as the outliers. Table~\ref{tab:dataset} shows the numbers of instances, features, clusters and outliers of these data sets.

\noindent\textbf{Competitive Methods.} K-means and K-means{-}{-} are used for comparisons. For our COR algorithm, 100 basic partitions are generated via K-means by different cluster numbers from 2 to $2K$, then K-means{-}{-} is employed with the distance function in Eq.~\eqref{the:second} for the partition and outliers. Note that K-means{-}{-} and COR are fed with $K$ and $o$ for fair comparisons, which are true numbers of clusters and outliers, respectively. For K-means, we set the cluster number as the true number plus one, the cluster found by K-means with the smallest size is regarded as the outlier set. Codes of K-means, K-measn{-}{-} and COR are implemented by MATLAB. Each algorithm runs 20 times, and returns the average result and standard deviation. Moreover, several classical outlier detection methods including density-based LOF\cite{Breunig00SIR}, COF\cite{Tang02PKDD}, distance-based LODF~\cite{Zhang09PKDD}, angle-based FABOD~\cite{Pham12KDD}, ensemble-based iForest~\cite{Liu08ICDM}, eigenvector-based OPCA~\cite{Lee13TKDE}, cluster-based TONMF~\cite{Kannan17SDM} are also involved as the competitive methods to evaluate the outlier detection performance\footnote{The codes of outlier detection methods can be found at \url{https://github.com/dsmi-lab-ntust/AnomalyDetectionToolbox} and \url{https://github.com/ramkikannan/outliernmf}.}. $o$ points with the largest scores by these methods are regarded as outliers. For the outlier detection methods, some default settings in the original papers are used for stable results. The number of nearest neighbors in LOF, COF, LODF and FABOD is set to 50; the sub-sampling size and the number of trees in iForest are 200 and 100; the forgetting number is set to 0.1 in OPCA; the rank and two parameters in TONMF are 10, 10 and 0.1, respectively.

\begin{table}[t!]
  \caption{Characteristics of data sets}\vspace{-0.2cm}
  \scriptsize
  \centering
  \begin{tabular}{lccccc}
  \hline
  Data set & Type & \#instance  & \#feature & \#cluster & \#outlier\\
  \hline
  \textit{ecoli} & Gene & 336 & 7 & 5 & 9\\
  \textit{yeast} & Gene & 1484 & 8 & 4 & 185\\
  \textit{caltech}& Image & 1415 & 4096 & 4 & 67\\
  \textit{sun09} & Image & 3282 & 4096 & 3 & 50 \\
  \textit{fbis} & Text & 2463 & 2000 & 10 & 332\\
  \textit{k1b} & Text & 2340 & 21839 & 5 & 60\\
  \textit{re0} & Text & 1504 & 2886 & 5 & 218 \\
  \textit{re1} & Text & 1657 & 3758 & 6 & 527\\
  \textit{tr11} & Text & 414 & 6129 & 4 & 87 \\
  \textit{tr23} & Text & 204 & 5832 & 3 & 32\\
  \textit{wap} & Text &1560 & 8460 & 10 & 251\\
  \textit{glass} & UCI & 214 & 9 & 3 & 39\\
  \textit{shuttle}&  UCI & 58000 & 9 & 3 & 244\\
  \textit{kddcup} & UCI & 494021 & 38 & 3 & 54499\\
  \hline
  \end{tabular} \label{tab:dataset}
\end{table}

\begin{table*}[t!]
  \caption{Performance of clustering with outlier removal via different algorithms (\%)}\vspace{-0.2cm}
  \scriptsize
  \centering
  \resizebox{1\textwidth}{!}{
  \begin{tabular}{l|ccc|ccc|ccc|ccc}
  \hline
  \multirow{2}{*}{Data set} & \multicolumn{3}{c|}{NMI} & \multicolumn{3}{c|}{Rn} & \multicolumn{3}{c|}{Jaccard} & \multicolumn{3}{c}{F-measure}\\
  \cline{2-13}
   & K-means & K-means{-}{-} & COR & K-means & K-means{-}{-} & COR & K-means & K-means{-}{-} & COR & K-means & K-means{-}{-} & COR \\
  \hline
  \textit{ecoli} & 62.27$\pm$2.78 & 61.81$\pm$2.37 & \textbf{63.16$\pm$1.76} & 56.6$\pm$11.23 & 52.62$\pm$10.33 & \textbf{61.68$\pm$8.74} & 2.34$\pm$2.02 & 45.76$\pm$12.78 & \textbf{47.37$\pm$3.75} & 4.50$\pm$3.71 & 61.58$\pm$14.6 & \textbf{64.21$\pm$3.24}\\
  \textit{yeast} & 19.91$\pm$0.77 & 15.81$\pm$1.52 & \textbf{20.41$\pm$1.08} & 14.40$\pm$0.72 & 11.85$\pm$1.93 & \textbf{18.07$\pm$2.04} & 4.92$\pm$1.33 & 14.38$\pm$6.14 & \textbf{50.47$\pm$1.45} & 9.34$\pm$2.45 & 24.69$\pm$8.92 & \textbf{67.07$\pm$1.29}\\
  \textit{caltech}& 69.13$\pm$9.92 & 65.34$\pm$11.76 & \textbf{82.69$\pm$7.04} & 47.65$\pm$15.48 & 53.26$\pm$24.94 & \textbf{71.25$\pm$18.18} & 8.27$\pm$11.41 & 30.36$\pm$15.45 & \textbf{97.19$\pm$1.39} & 13.59$\pm$17.88 & 44.37$\pm$19.84 & \textbf{98.57$\pm$0.72}\\
  \textit{sun09} & 19.67$\pm$0.21 & 10.76$\pm$1.41 & \textbf{21.50$\pm$1.17} & 18.61$\pm$0.21 & 9.11$\pm$1.69 & \textbf{20.29$\pm$1.91} & 1.84$\pm$0.09 & \textbf{3.27$\pm$0.44} & 2.27$\pm$0.22 & 3.62$\pm$0.16 & \textbf{6.34$\pm$0.81} & 4.44$\pm$0.42\\
  \textit{fbis} & 9.77$\pm$2.41 & 30.27$\pm$3.43 & \textbf{54.13$\pm$0.85} & -0.90$\pm$0.23 & 9.00$\pm$3.65 & \textbf{38.36$\pm$2.32} & 0.02$\pm$0.07 & 5.21$\pm$0.15 & \textbf{23.77$\pm$2.24} & 0.03$\pm$0.14 & 9.91$\pm$0.27 & \textbf{38.35$\pm$2.95}\\
  \textit{k1b} & 34.4$\pm$18.55 & 33.11$\pm$17.06 & \textbf{50.00$\pm$5.15} & 25.12$\pm$18.87 & 21.74$\pm$22.48 & \textbf{32.19$\pm$9.82} & 0.00$\pm$0.00 & 0.00$\pm$0.00 & \textbf{20.53$\pm$0.82} & 0.00$\pm$0.00 & 0.00$\pm$0.00 & \textbf{34.06$\pm$1.12}\\
  \textit{re0} & 17.14$\pm$3.06 & 15.28$\pm$2.78 & \textbf{32.12$\pm$2.76} & 8.86$\pm$2.80 & 8.74$\pm$4.54 & \textbf{21.76$\pm$3.83} & 3.18$\pm$2.38 & 8.82$\pm$0.68 & \textbf{28.50$\pm$1.20} & 6.07$\pm$4.45 & 16.20$\pm$1.15 & \textbf{44.34$\pm$1.44} \\
  \textit{re1} & 16.19$\pm$3.47 & 11.57$\pm$3.92 & \textbf{35.46$\pm$2.69} & 2.79$\pm$1.36 & 3.59$\pm$1.81 & \textbf{20.79$\pm$2.51} & 0.29$\pm$0.25 & 16.75$\pm$0.34 & \textbf{27.64$\pm$1.88} & 0.59$\pm$0.50 & 28.70$\pm$0.51 & \textbf{43.28$\pm$2.30}\\
  \textit{tr11} & 9.61$\pm$0.68 & 14.89$\pm$6.95 & \textbf{58.69$\pm$3.94} & 0.40$\pm$0.12 & 3.92$\pm$4.71 & \textbf{50.95$\pm$8.55} & 0.00$\pm$0.00 & 9.93$\pm$0.42 & \textbf{34.06$\pm$3.03} & 0.00$\pm$0.00 & 18.06$\pm$0.70 & \textbf{50.74$\pm$3.44} \\
  \textit{tr23} & 7.08$\pm$0.81 & 9.82$\pm$2.86 & \textbf{19.43$\pm$6.60} & -3.67$\pm$0.47 & 1.31$\pm$3.02 & \textbf{14.01$\pm$8.04} & 0.00$\pm$0.00 & 5.87$\pm$1.02 & \textbf{12.35$\pm$2.66} & 0.00$\pm$0.00 & 11.08$\pm$1.84 & \textbf{21.88$\pm$4.31}\\
  \textit{wap} & 40.51$\pm$2.85 & 22.75$\pm$10.42 & \textbf{48.31$\pm$2.47} & 12.03$\pm$2.31 & 6.33$\pm$6.33 & \textbf{29.24$\pm$7.40} & 0.60$\pm$0.51 & 10.98$\pm$0.31 & \textbf{22.01$\pm$1.30} & 1.18$\pm$0.99 & 19.78$\pm$0.50 & \textbf{36.06$\pm$1.74}\\
  \textit{glass} & 31.35$\pm$5.90 & 33.48$\pm$3.78 & \textbf{35.88$\pm$3.94} & 20.03$\pm$3.50 & 23.47$\pm$2.09 & \textbf{24.86$\pm$1.72} & 8.23$\pm$5.41 & 24.00$\pm$8.28 & \textbf{32.67$\pm$2.87} & 14.8$\pm$8.84 & 37.97$\pm$11.59 & \textbf{49.18$\pm$3.24}\\
  \textit{shuttle}& 11.06$\pm$12.49 & 22.95$\pm$3.21 & \textbf{30.74$\pm$5.41} & 18.18$\pm$22.67 & 27.73$\pm$5.71 & \textbf{47.40$\pm$12.89} & 0.00$\pm$0.00 & 5.39$\pm$0.00 & \textbf{5.58$\pm$0.93} & 0.00$\pm$0.00 & 10.22$\pm$0.00 & \textbf{10.56$\pm$1.73}\\
  \textit{kddcup} & 1.41$\pm$0.05 & 67.04$\pm$5.18 & \textbf{86.12$\pm$0.60} & 0.04$\pm$0.00 & 71.40$\pm$9.81 & \textbf{93.91$\pm$0.85} & 0.01$\pm$0.00 & 15.06$\pm$3.26 & \textbf{15.98$\pm$0.63} & 0.02$\pm$0.00 & 26.03$\pm$5.55 & \textbf{27.55$\pm$0.96}\\
  \hline
  Average & 24.96 & 29.63 & \textbf{45.62} & 15.72 & 21.72 & \textbf{38.91} & 2.12 & 12.98 & \textbf{30.03} & 3.84 & 22.49  & \textbf{42.16}\\
  Score & 58.48 & 62.54 & \textbf{100} & 42.47 & 49.05 & \textbf{100} &8.56 & 54.34 & \textbf{97.82} & 10.03 & 58.28 & \textbf{97.87}\\
  \hline
  \end{tabular}} \label{tab:performance}
\end{table*}

\begin{figure*}[t!]
  \centering
    \subfigure[\scriptsize \textit{caltech}]{
    \includegraphics[width=0.23\textwidth]{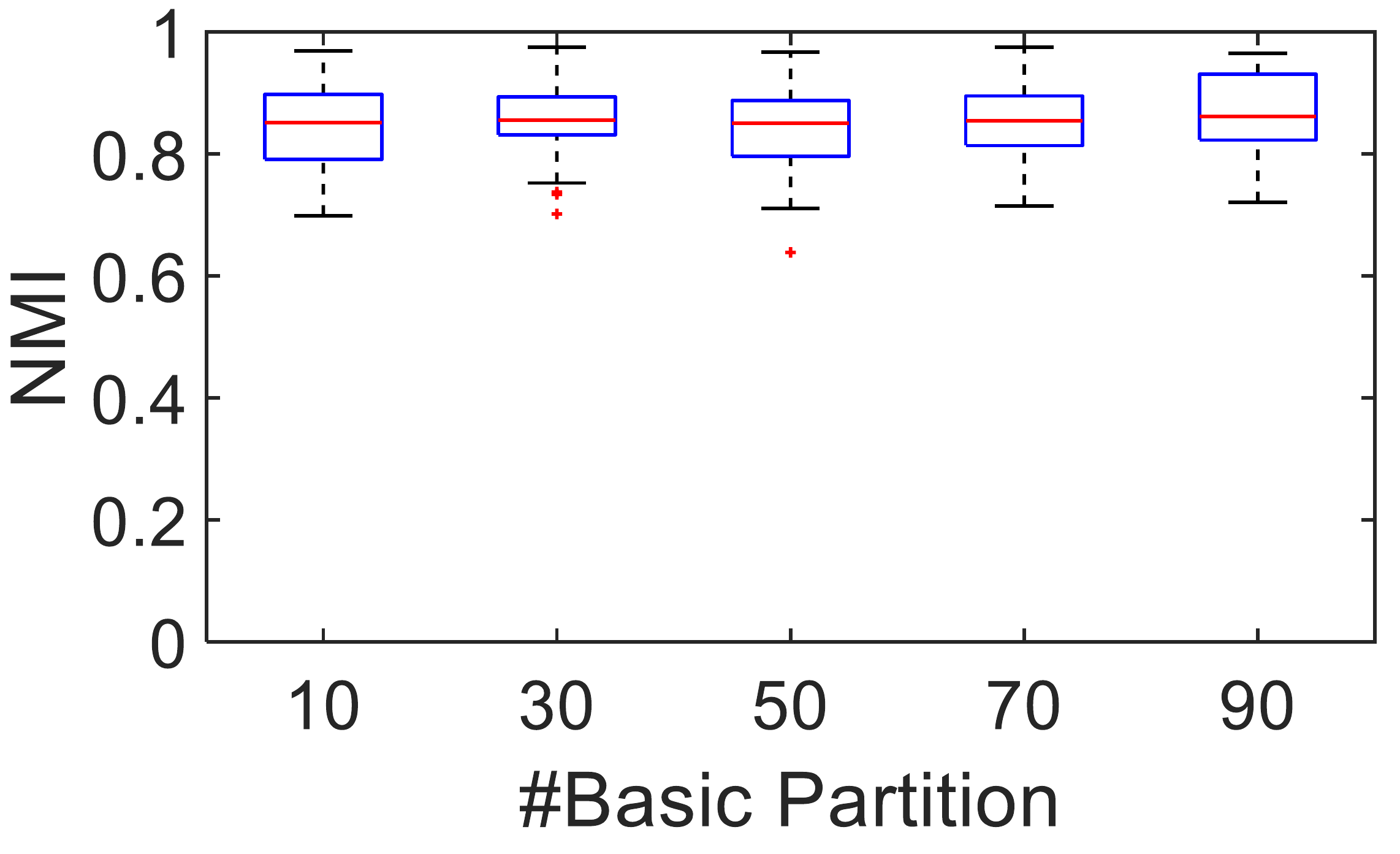}
  }
   \subfigure[\scriptsize \textit{caltech}]{
    \includegraphics[width=0.23\textwidth]{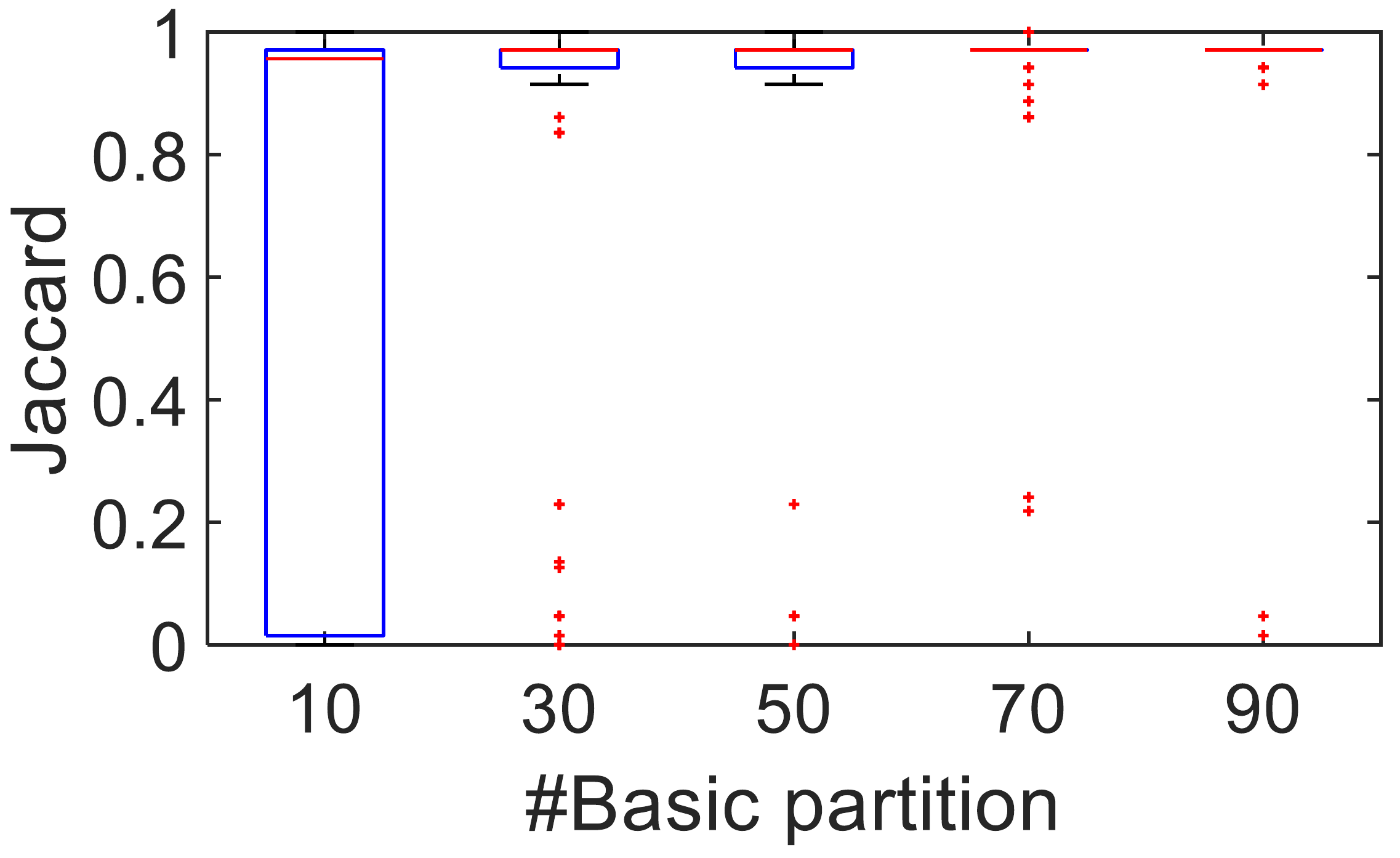}
  }
  \subfigure[\scriptsize \textit{fbis}]{
    \includegraphics[width=0.23\textwidth]{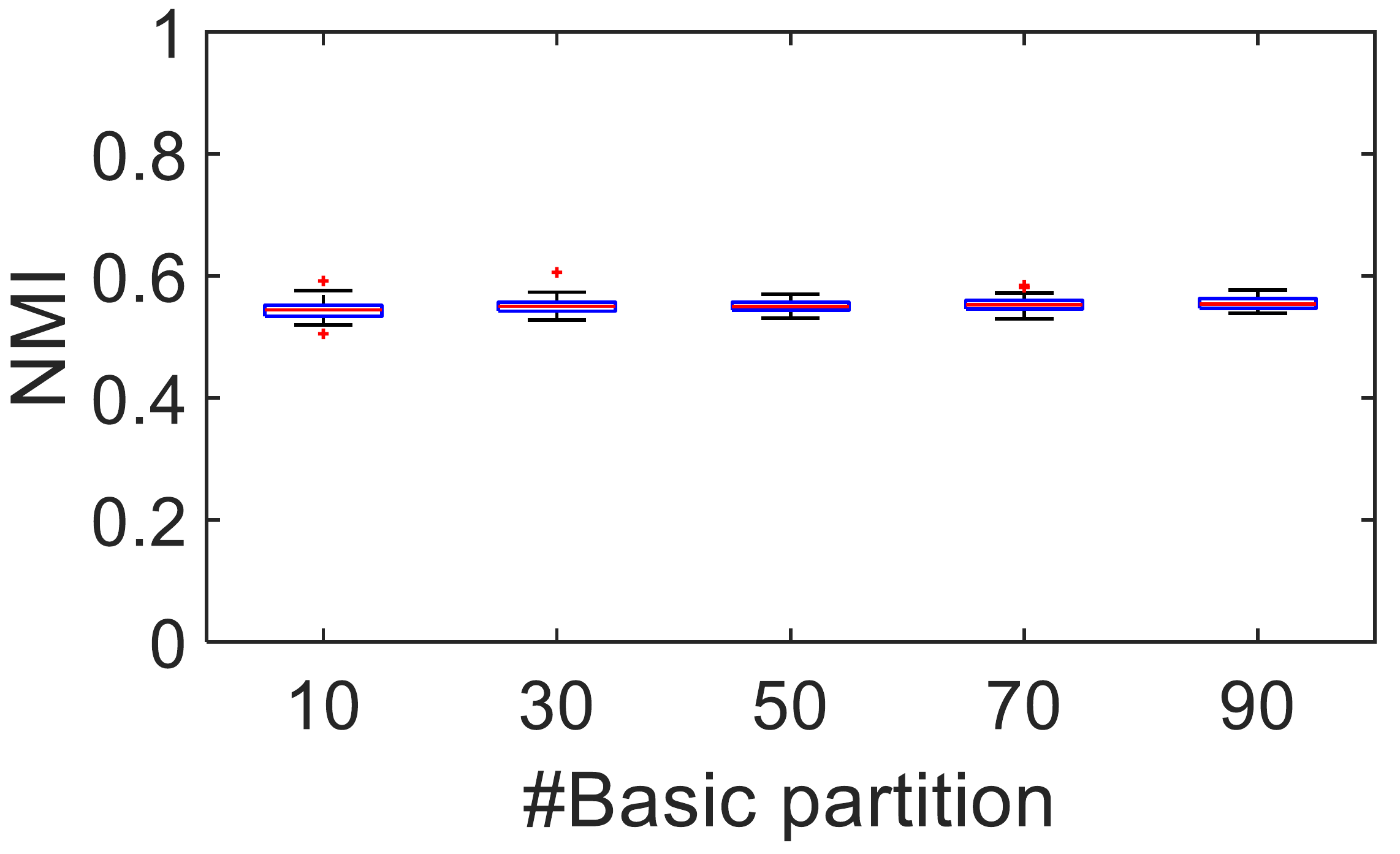}
  }
   \subfigure[\scriptsize \textit{fbis}]{
    \includegraphics[width=0.23\textwidth]{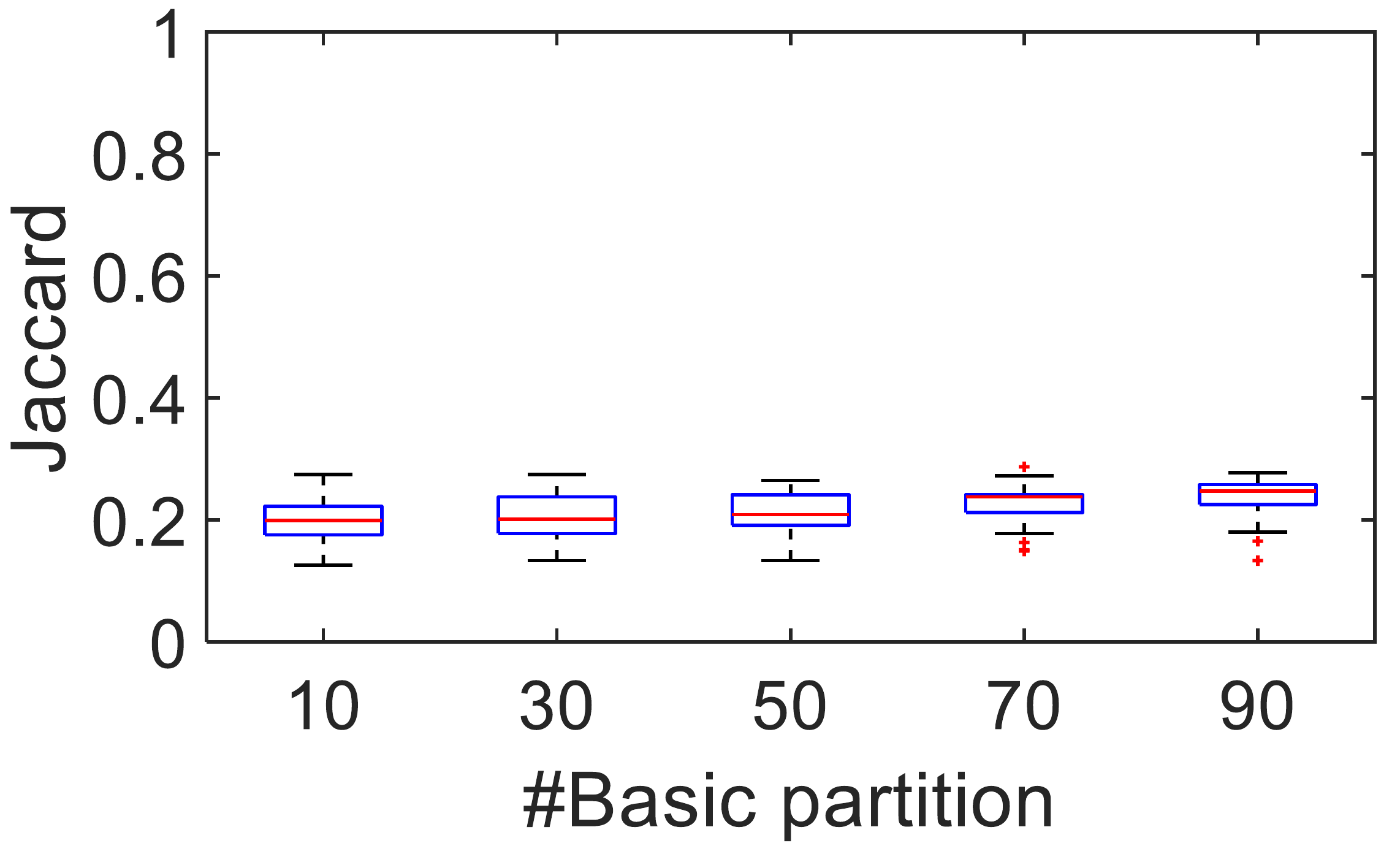}
  }
  \caption{Performance of COR with different numbers of basic partitions on \textit{caltech} and \textit{fbis}.}
  \label{fig:bp}
\end{figure*}

\noindent\textbf{Validation metric.} Although the clustering with outlier removal is an unsupervised task, we can still apply the ground truth to evaluate the performance with label information. Since we focus on the jointly clustering and outlier detection, four metrics are employed to evaluate the performance in terms of cluster validity and outlier detection. The outlier set is regarded as a special cluster in the ground truth.

Normalized Mutual Information ($NMI$) and Normalized Rand Index ($R_n$) are two widely used external measurements for cluster validity~\cite{Wu09KDD}. $NMI$ measures the mutual information between resulted cluster labels and ground truth labels, followed by a normalization operation, while $R_n$ measures the similarity between two partitions in a statistical way. They can be computed as follows.
\begin{equation}\nonumber
NMI = \frac{\sum_{i,j} n_{ij}\log \frac{n\cdot n_{ij}}{n_{i+}\cdot n_{+j}}}{\sqrt{(\sum_{i} n_{i+}\log \frac{n_{i+}}{n})(\sum_{j} n_{j+}\log \frac{n_{+j}}{n})}},
\end{equation}
\begin{equation}\nonumber
R_n = \frac{\sum_{i,j}\binom{n_{ij}}{2} -\sum_{i}\binom{n_{i+}}{2}\cdot \sum_{j} \binom{n_{+j}}{2}/\binom{n}{2}}{\sum_{i}\binom{n_{i+}}{2}/2 + \sum_{j} \binom{n_{+j}}{2}/2 -\sum_{i}\binom{n_{i+}}{2}\cdot \sum_{j} \binom{n_{+j}}{2}/\binom{n}{2}},
\end{equation}
where $n_{ij}$, $n_{i+}$, $n_{+j}$ are the co-occurrence number and cluster size of $i$-th and $j$-th cluster in the obtained partition and ground truth, respectively.

Jaccard index and F-measure are designed for the binary classification, which are employed to evaluate the outlier detection. They can be computed as follows.
\begin{equation}\nonumber
Jaccard = \frac{|O\cap O^*|}{|O\cup O^*|},
\end{equation}
\begin{equation}\nonumber
F{-}measure = 2*\frac{\textrm{precition}\cdot \textrm{recall}}{\textrm{precition}+ \textrm{recall}},
\end{equation}
where $O$ and $O^*$ are the outlier sets by the algorithm and ground truth, respectively, and F-measure is the harmonic average of the precision and recall for outlier class.

To evaluate the overall performance on all used data sets, we propose a score as follows.
\begin{equation}\nonumber
 sorce(A_i)=\sum_{j}\frac{P(A_i,D_j)}{\max_{i}P(A_i,D_j)},
\end{equation}
where $P(A_i,D_j)$ denotes the performance of algorithm $A_i$ on data set $D_j$ in terms of some metric.

Note that these four metrics and the score are positive measurements, i.e, a larger value means better performance. Although $R_n$ is normalized, it can still be negative, which means that the partition is even worse than random label assignment.

\textbf{Environment.} All experiments were run on a PC with an Intel Core i7-5930K@3.50 GHz and a 64 GB DDR3 RAM.

\begin{table*}[t!]
  \caption{Performance of outlier detection via different algorithms (\%)}\vspace{-0.2cm}
  \scriptsize
  \centering
  \resizebox{1\textwidth}{!}{
  \begin{tabular}{l|ccccccccc|ccccccccc}
  \hline
  \multirow{2}{*}{Data set} & \multicolumn{9}{c|}{Jaccard} & \multicolumn{9}{c}{F-measure}\\
  \cline{2-19}
   & LOF & COF & LDOF & FABOD & iForest & OPCA & TONMF & K-means{-}{-} &COR & LOF & COF & LDOF & FABOD & iForest & OPCA& TONMF& K-means{-}{-}& COR\\
  \hline
  \textit{ecoli}  & 20.00 & 38.46 & 5.88 & 20.00 & 38.28 & 5.88 & 0.00 & 45.76 & \textbf{47.37} & 33.33 & 55.56 & 11.11 & 33.33 & 55.56 & 11.11 & 0.00& 61.58& \textbf{64.21}\\
  \textit{yeast}  & 11.45 & 11.45 & 5.11 & 13.85 & 23.75 & 26.71 & 8.66 & 14.38 & \textbf{50.47} & 20.54 & 20.54 & 9.73 & 24.32 & 38.38 & 42.16 & 8.11&24.69 &\textbf{67.07}\\
  \textit{caltech}& 2.29  & 0.75  & 1.52 & 8.06 & 27.62 & 0.00 & 0.00 & 30.36 & \textbf{97.19}& 4.48 & 1.49 & 2.99 & 14.93 & 43.28 & 0.00 & 1.49&44.37&\textbf{98.57}\\
  \textit{sun09}  & 1.01  & 2.04  & 0.00 & 2.04 & 2.04 & 0.00 & 0.00 & \textbf{3.27} & 2.27 & 2.00 & 4.00 & 0.00 & 4.00 & 4.00 & 0.00 & 6.00&\textbf{6.34}&4.44\\
  \textit{fbis}   & 8.32  & 5.56  & 4.90 & 6.41 & 5.40 & 4.40 & 8.32 & 5.21 & \textbf{23.77} & 15.36 & 10.54 & 9.34 & 12.05 & 10.24 & 8.43 & 15.36&9.91& \textbf{38.35}\\
  \textit{k1b}    & 0.00  & 0.00  & 0.00 & 0.84 & 0.00 & 0.00 & 1.69 & 0.00 &\textbf{20.53} & 0.00 & 0.00 & 0.00 & 1.67 & 0.00 & 0.00 & 3.33&0.00&\textbf{34.06}\\
  \textit{re0}    & 2.59  & 5.31  & 3.07 & 6.34 & 2.83 & 11.79 & 7.13 & 8.82 &\textbf{28.50} & 5.05 & 10.09 & 5.96 & 11 93 & 5.50 & 21.10 & 13.30&16.20&\textbf{44.34}\\
  \textit{re1}    & 21.85 & 15.44 & 15.19 & 18.83 & 16.85 & 17.77 & 16.98 & 16.75 & \textbf{27.64} & 35.86 & 26.76 & 26.38 & 31.69 & 28.84 & 30.17 &29.03&28.70& \textbf{43.28}\\
  \textit{tr11}   & 10.13 & 8.75  & 19.18 & 10.83 & 8.75 & 8.75 &  12.99 & 9.93 &\textbf{34.06} & 18.39 & 16.09 & 32.18 & 19.54 & 16.09 & 16.09 & 22.99&18.06&\textbf{50.74}\\
  \textit{tr23}   & 4.92  & 4.92  & 6.67 & 10.34 & 6.67 & 1.59 & 10.34 & 5.87 & \textbf{12.35} & 9.37 & 9.37 & 12.50 & 18.75 & 12.50 & 3.12 &18.75& 11.08&\textbf{21.88}\\
  \textit{wap}    & 10.82 & 12.30 & 6.36 & 12.81 & 11.31 & 6.58 & 7.49 & 10.98 & \textbf{22.01} & 19.52 & 21.91 & 11.95 & 22.71 & 23.75 & 12.35 &13.94& 19.78&\textbf{36.06}\\
  \textit{glass}  & 16.42 & 36.84 & 4.00 & 25.81 & 13.04 & 14.71 & 0.00 & 24.00 &\textbf{32.67} & 28.21 & 53.85 & 76.90 & 41.03 & 23.08 & 25.64 & 0.00& 37.97&\textbf{49.18}\\
  \textit{shuttle}& 12.44 & \textbf{12.96} & 0.21 & 7.25 & 1.46  &3.61& 0.00 & 5.39&5.58 & 22.13 & \textbf{22.95} & 0.41 & 13.52 & 2.87 & 6.97 & 0.00& 10.22&10.56\\
  \textit{kddcup} & N/A & N/A & N/A & N/A & \textbf{21.22} & 15.66 & 8.66 & 15.06& 15.98 & N/A & N/A & N/A & N/A & \textbf{35.01} & 27.08 &15.94& 26.03&27.55\\
  \hline
  Average & 9.40 & 11.91 & 5.55 & 11.03 & 12.80 & 8.39 & 5.83 & 12.98 & \textbf{30.03} & 16.48 & 19.47 & 15.34 & 19.19 & 21.36 & 14.59 &10.59& 22.49&\textbf{42.16}\\
  Score & 34.33 & 41.88 & 18.73 & 40.10 & 42.14 & 28.54 & 30.72 & 48.15 & \textbf{91.17} & 37.09 & 42.80 & 27.79 & 43.02 & 46.70 & 31.48 & 33.60 & 51.13 & \textbf{89.91} \\
  \hline
  \multicolumn{19}{l}{Note: We omit the standard deviations due to the determinacy of most outlier detection methods. N/A means failure to deliver results due to out-of-memory on a PC machine with 64G RAM. }\\
  \end{tabular}} \label{tab:detection}
\end{table*}

\subsection{Algorithmic Performance}
Here we evaluate the performance of COR by comparing with K-means{-}{-} and outlier detection methods. Table~\ref{tab:performance} shows the performance of clustering with outlier removal via K-means, K-means{-}{-} and COR. There are three obvious observations. (1) few outliers can easily destroy the whole cluster structure. This point can be verified from the fact that K-means delivers poor clustering results on \textit{fbis}, \textit{tr23} and \textit{kddcup} in terms of NMI and Rn. Moreover, K-means fails to capture the outliers by simply increasing the cluster number. (2) K-means{-}{-} jointly learns the cluster structure and detects the outliers, which alleviates the negative impact of outliers on the clusters and achieves better performance over K-means on the average level. Although K-means{-}{-} slightly outperforms COR on \textit{sun09} in terms of outlier detection, the cluster structure provided by K-means{-}{-} is much worse than COR, even K-means clustering. (3) COR exceeds K-means and K-means{-}{-} by a large margin in both cluster analysis and outlier detection. For example, COR gains more than 10\%, 20\% and 40\% improvements by cluster validity over rivals on \textit{caltech}, \textit{fbis} and \textit{tr11}, respectively. Moreover, COR also provides better outlier detection results. On \textit{yeast} and \textit{caltech}, there exists more than 30\%, 50\% gains over K-means{-}{-}; especially, on \textit{k1b}, COR achieves 25.53 and 34.06 in terms of Jaccard and F-measure; however, K-means{-}{-} fails to detect any outliers. Recall that COR is in essence K-means{-}{-} on the binary matrix $[B\ \widetilde{B}]$. The huge improvements result from the partition space, where defines the concept of clusters and achieves the joint consensus clustering and outlier removal. From the score, COR significantly outperforms K-means and \mbox{K-means{-}{-}} in terms of all four metrics. Since COR is conducted in the partition space, we also compare with K-means-based Consensus Clustering (KCC)~\cite{Wu15TKDE} with the same basic partitions by adding one more cluster to capture the outliers. Due to the limited page, we report that on the average level, KCC delivers the competitive cluster results, where COR slightly outperforms KCC by 1.21\% and 3.95\% in terms of NMI and Rn. Unfortunately, KCC fails to detect any outliers on all the datasets.

\begin{table}[t!]
  \caption{Execution time by second}\vspace{-0.2cm}
  \scriptsize
  \centering
  \begin{tabular}{l|rrrrr}
  \hline
  Method & \textit{sun09} & \textit{k1b} & \textit{wap} & \textit{shuttle} & \textit{kddcup}\\
  \hline
  K-means & 1.12 & 4.55 & 1.25 & 0.22 & 0.62\\
  LOF & 65.16 & 150.38 & 26.81 & 11.93 & N/A \\
  COF & 79.50 & 154.02 & 30.18 & 181.45 & N/A\\
  LDOF & 277.25 & 2638.97 & 903.43 & 246.87 & N/A \\
  FABOD & 567.47 & 5373.76 & 1811.28 & 495.43 & N/A\\
  iForest & 12.55 & 12.88 & 8.53 & 165.42 & 1455.41\\
  OPCA & 0.40 & 6.18 & 1.75 & 0.30 & 2.51 \\
  TONMF&7.87 & 31.76 & 7.67 & 1.18 & 18.17\\
  K-means{-}{-} & 3.56 & 65.28 & 12.73 & 0.33 & 5.98 \\
  \hline
  BP & 52.86 & 121.95 & 36.58 & 5.09 & 5.55\\
  COR & 2.31 & 0.15 & 0.19 & 0.57 & 2.89\\
  \hline
  \multicolumn{6}{l}{Note: BP shows the time for generating 100 basic partitions. }\\
  \end{tabular} \label{tab:time}
\end{table}

Beyond K-means and K-means{-}{-}, we also compare COR with several outlier detection methods. Table~\ref{tab:detection} shows the performance of outlier detection in terms of Jaccard and F-measure. These algorithms are based on different assumptions including density, distance, angle, ensemble, eigenvector and clusters, and sometimes effective on certain data set. For example, COF and iForest get the best performance on \textit{shuttle} and \textit{kddcup}, respectively. However, in the most cases, these competitors show the obvious disadvantages in terms of performance. The reasons are complicated, but the original space and unsupervised parameter setting might be two of them. For TONMF, there are three parameters as the inputs, which are difficult to set without any knowledge from domain experts. Differently, COR requires two straightforward parameters, and benefits from the partition space and joint clustering with outlier removal, which brings the extra gains on several data sets. On \textit{shuttle} and \textit{kddcup}, COR does not deliver the results as good as the outlier detection methods. In the next subsection, we further improve the performance of COR via different basic partition generation strategy.

Next we continue to evaluate these algorithms in terms of efficiency. Table~\ref{tab:time} shows the execution time of these methods on five large-scale or high-dimensional data sets. Generally speaking, the density-based, distance-based and angle-based methods become struggled on high-dimensional data sets, especially FABOD is the most time consuming method, while the cluster-based methods including TONMF, K-means{-}{-} are relatively fast. It is worthy to note that the density-based, distance-based and angle-based methods require to calculate the nearest neighbor matrix, which takes huge space complexity and fails to deliver results on large-scale data sets due to out-of-memory on a PC machine with 64G RAM. For COR, the time complexity is roughly linear to the number of instances; moreover, COR is conducted on the binary matrix, rather than the original feature space. Thus, COR is also suitable for high-dimensional data. On \textit{k1b}, COR only takes 0.15 seconds, over 400 times faster than K-means{-}{-}. Admittedly, COR requires a set of basic partitions as the input, which takes the extra execution time. In Table~\ref{tab:time}, we report the execution time of generating 100 basic partitions as well. This process can be further accelerated by parallel computing. Even taking the time of generating basic partition, COR is still much faster than the density-based, distance-based and angle-based outlier detection methods.

\begin{figure}[t!]
  \centering
    \subfigure[\scriptsize \textit{shuttle}]{
    \includegraphics[width=0.22\textwidth]{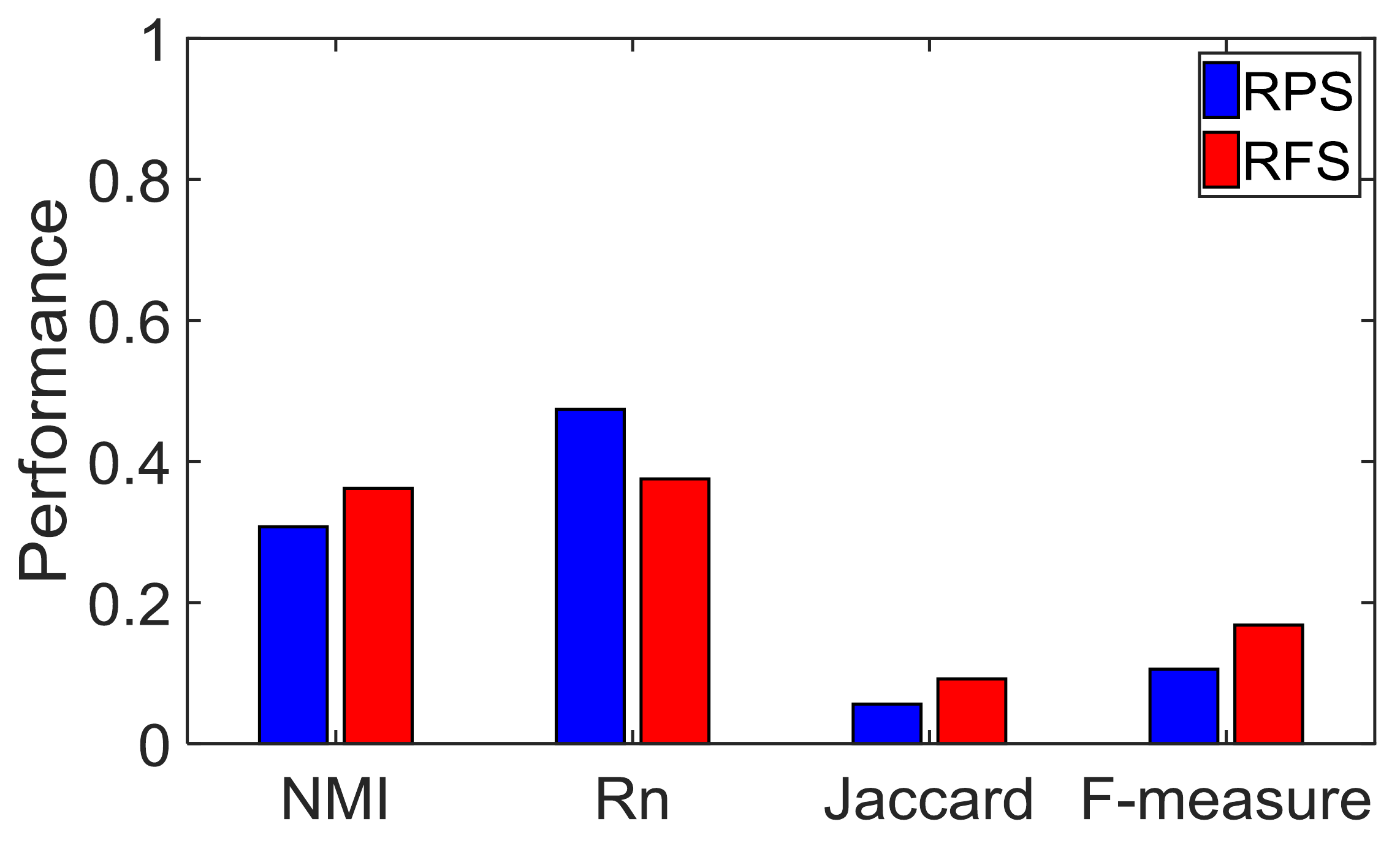}
  }\hspace{-2mm}
   \subfigure[\scriptsize \textit{kddcup}]{
    \includegraphics[width=0.22\textwidth]{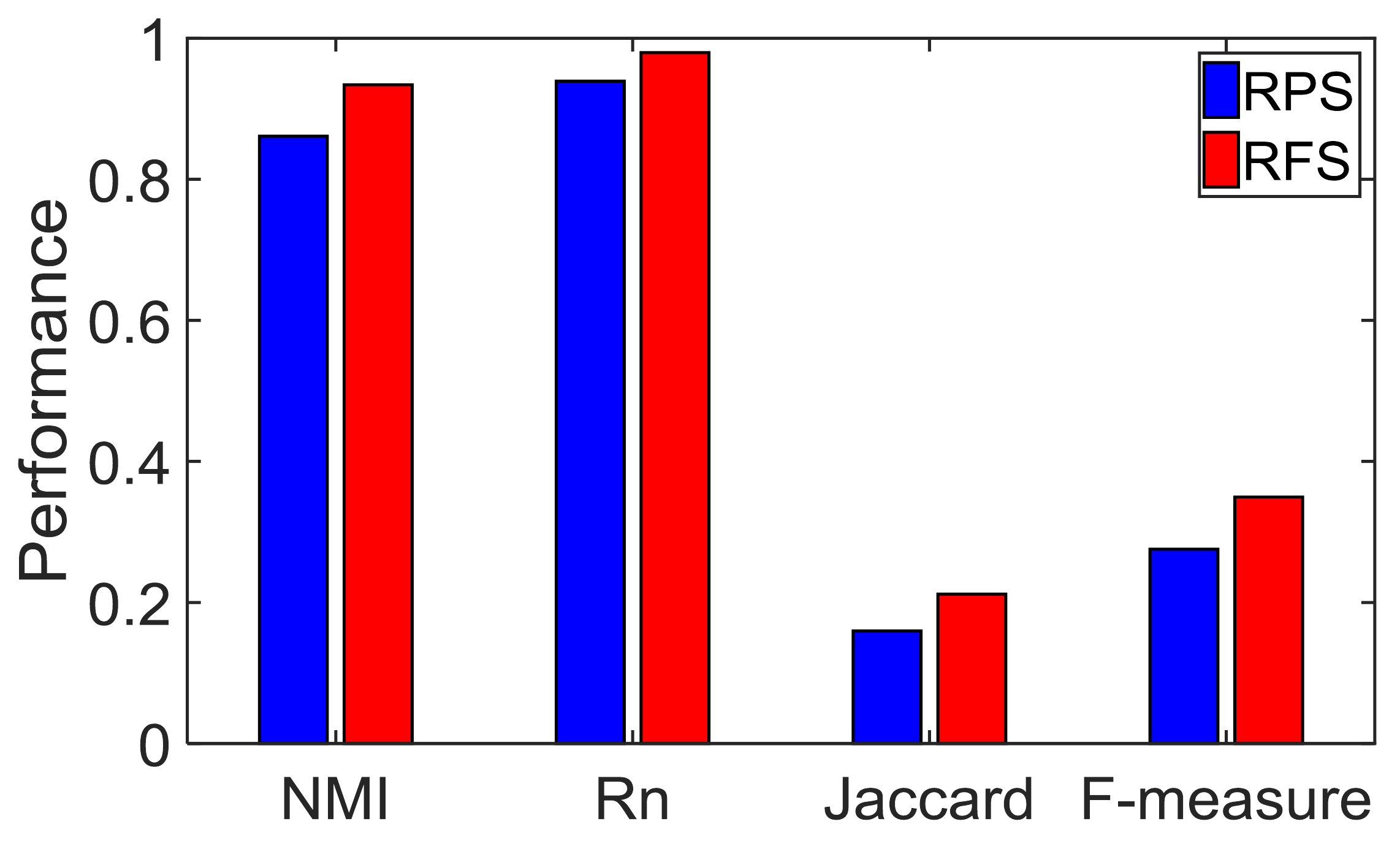}
  }
  \caption{Performance of COR with different basic partition generation strategies.}
  \label{fig:rfs}
\end{figure}

\begin{figure*}[t!]
  \centering
    \subfigure[\scriptsize \textit{Chinese flight trajectories}]{
    \includegraphics[width=0.23\textwidth]{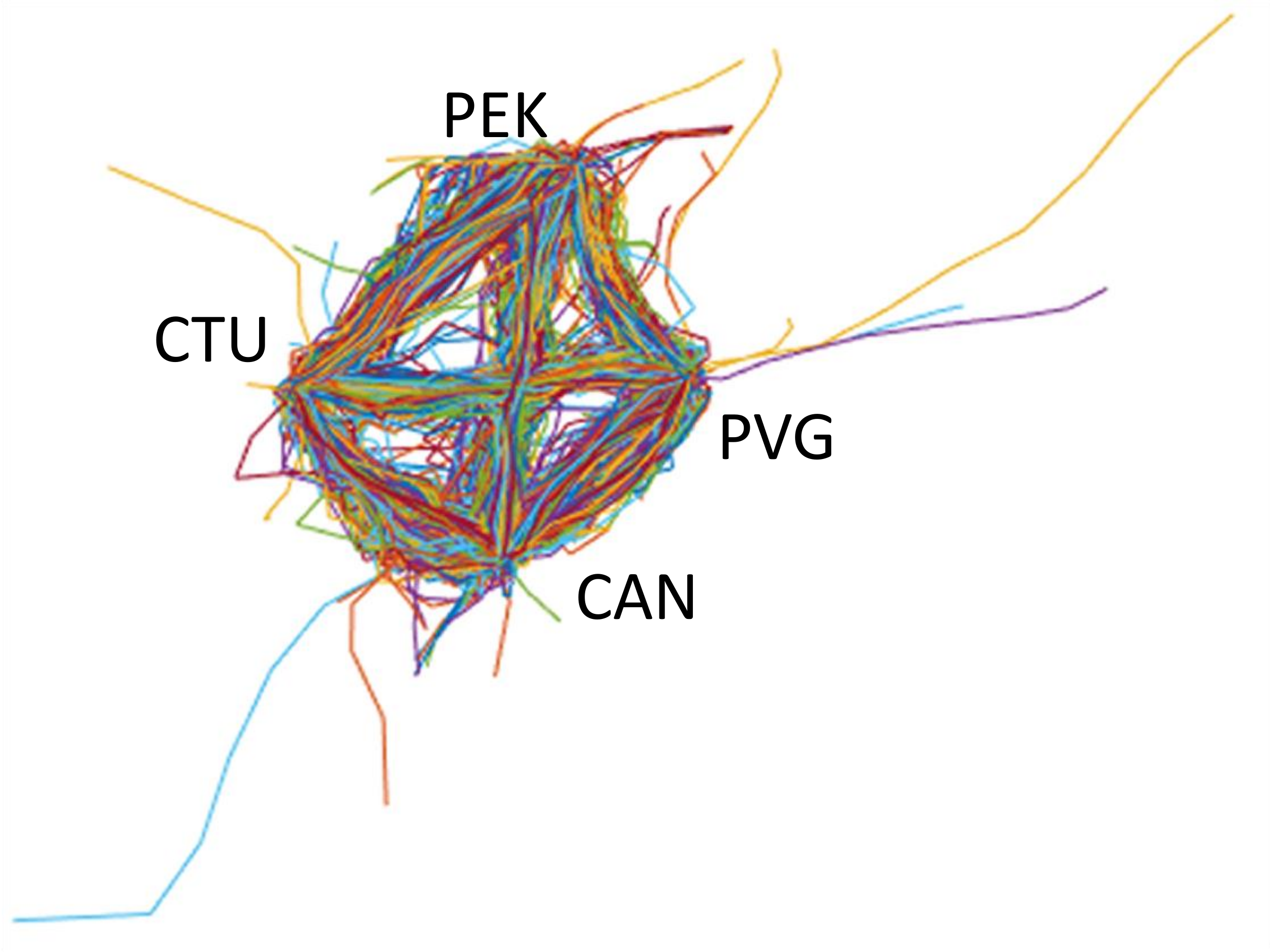}\label{fig:f-chinese}
  }
   \subfigure[\scriptsize \textit{Outlier trajectories in China}]{
    \includegraphics[width=0.23\textwidth]{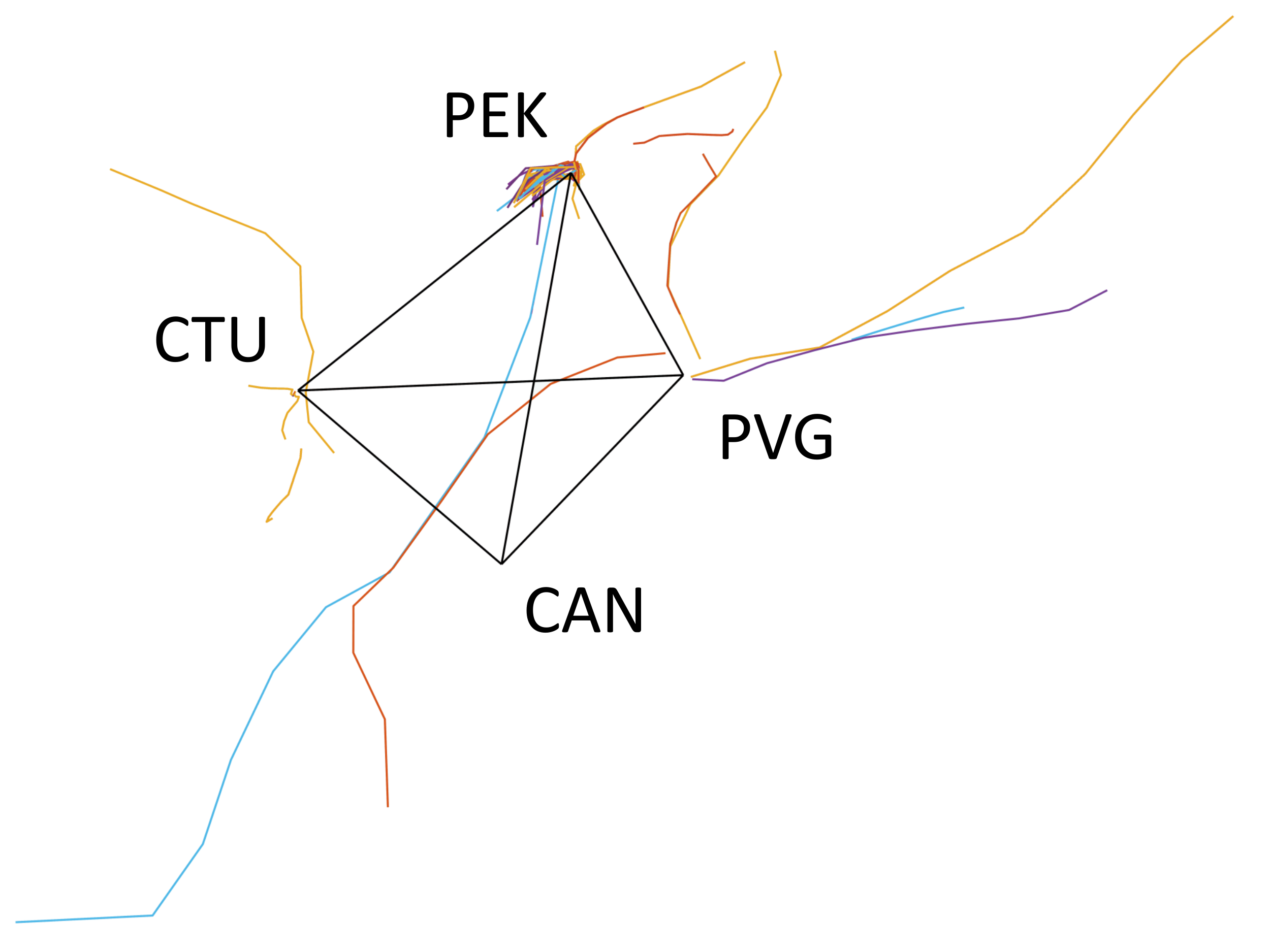}\label{fig:o-chinese}
  }
  \subfigure[\scriptsize \textit{US flight trajectories}]{
    \includegraphics[width=0.23\textwidth]{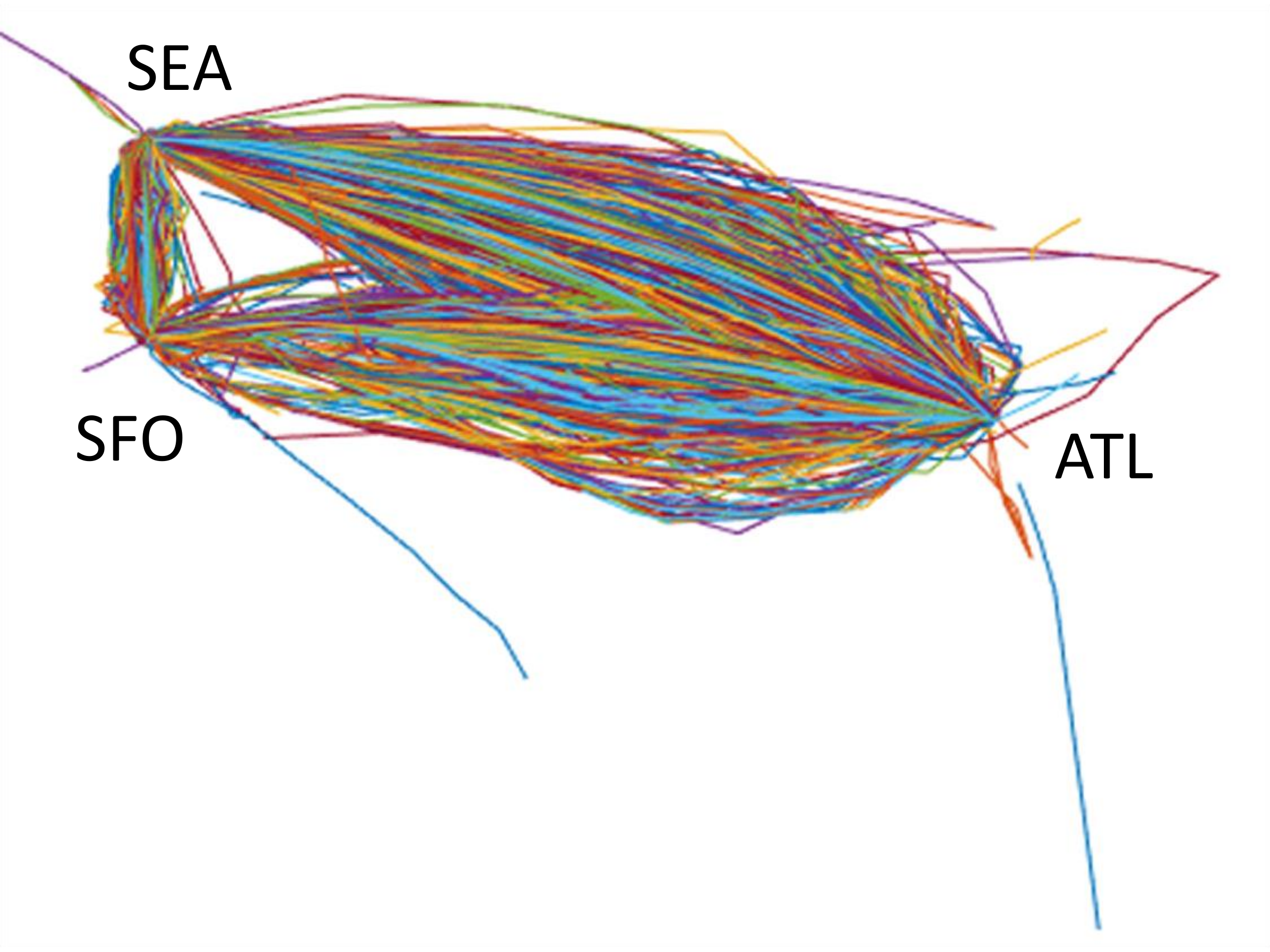}\label{fig:f-us}
  }
   \subfigure[\scriptsize \textit{Outlier trajectories in US}]{
    \includegraphics[width=0.23\textwidth]{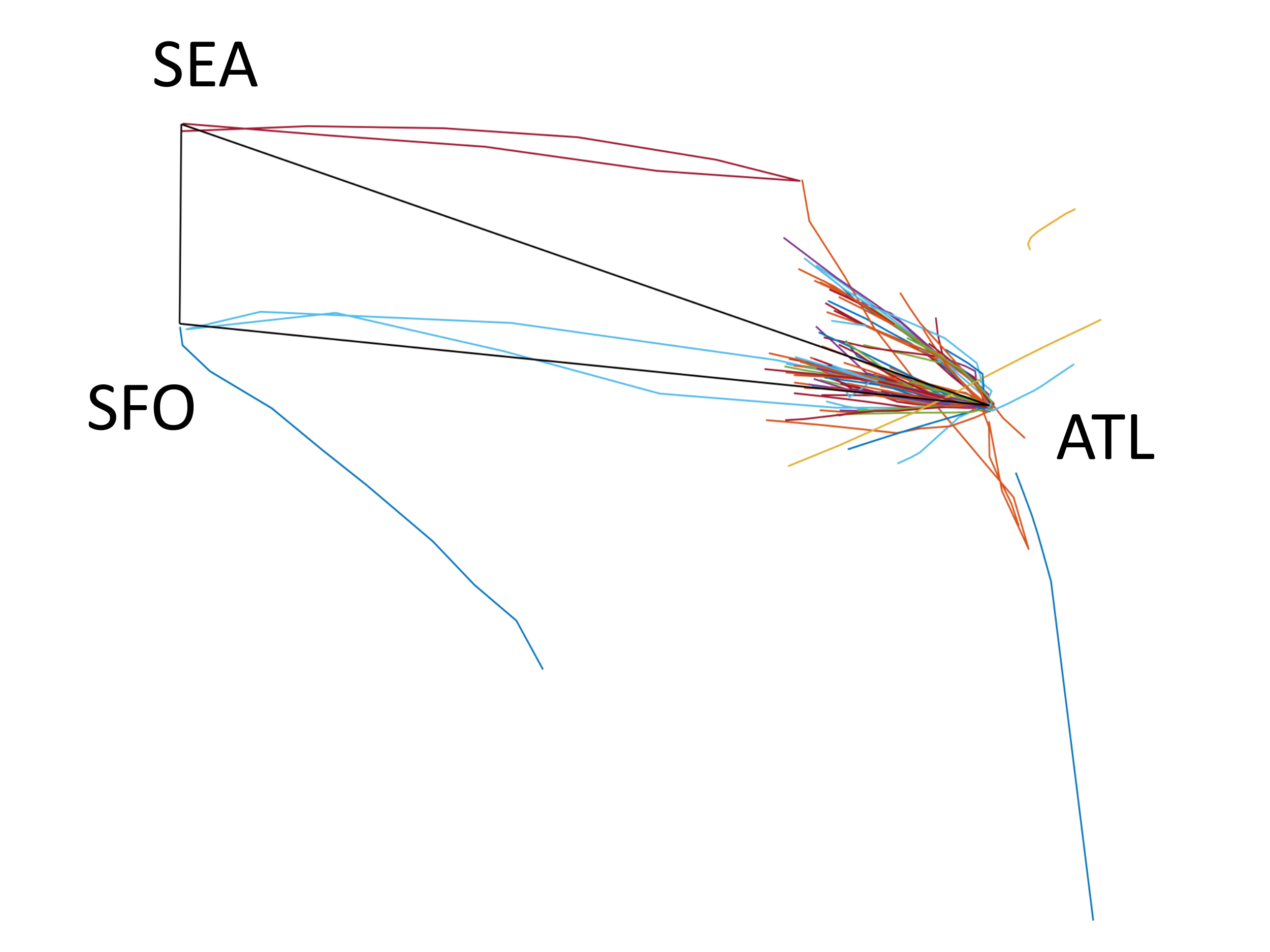}\label{fig:o-us}
  }
  \caption{Chinese and US flight trajectories. (a) \& (c) show the flight trajectories and (b) \& (d) demonstrate the outlier trajectories detected by COR. }
  \label{fig:flight}
\end{figure*}

\subsection{Factor Exploration}
In this subsection, we provide further analyses on the factors inside COR, the number of basic partitions and the basic partition generation strategy.

In consensus clustering, the performance of clustering goes up with the increase of basic partitions~\cite{Wu15TKDE,Liu16KDD,Liu17DMKD}. Similarly, we test COR with different numbers of basic partitions. Figure~\ref{fig:bp} shows the boxplot of the performance of COR with 10, 30, 50, 70 and 90 basic partitions on \textit{caltech} and \textit{fbis} in terms of NMI and Jaccard. For a certain number of basic partitions, we generate 100 sets of basic partitions and run COR for the boxplot. From Figure~\ref{fig:bp}, we have that COR delivers high quality partitions even with 10 basic partitions, and that for outlier detection, the performance slightly increases with more basic partitions and stabilizes in a small region. Generally speaking, 30 basic partitions are enough for COR to deliver a good result.

So far, we employ the Random Parameter Selection (RPS) strategy to generate basic partitions, which employs K-means clustering with different cluster numbers. In fact, Random Feature Selection (RFS) is another widely strategy to generation basic partitions, which randomly selects partial features for K-means clustering. In the following, we evaluate the performance of COR with RFS. Here we set the random feature selection ratio to be 50\% for 100 basic partitions. Figure~\ref{fig:rfs} shows the performance of COR with different basic partition generation strategies on \textit{shuttle} and \textit{kddcup}. RFS achieves some improvements over RPS on these two data sets with different metrics, except on \textit{shuttle} in terms of Rn. This indicates that RFS is helpful to alleviate the negative impact of noisy features, and further produces high quality basic partitions for COR. It is worthy to note that COR with RFS on \textit{kddcup} achieves 21.18 and 34.95 in terms of Jaccard and F-measure, which exceeds the one with RPS over 5\% and 7\%, and competes with iForest. This means that COR with RFS gets the competitive performance with the best rival on \textit{kddcup}, and it is over 170 times faster than iForest.

\subsection{Application on Trajectory Detection}
Finally, we evaluate our COR in the real-world application on outlier trajectory detection. The data come from Flight Tracker\footnote{\url{https://www.flightradar24.com}.}, including flightID, flightNum, timestamp, latitude, longitude, height, departure airport, arrival airport and other information. We employ the API to request the flight trajectory every 5 minutes, and collect one-year data from October, 2016 to September, 2017 all over the world. After the data processing, we organize the data with each row representing one flight with evolutional latitude and longitude. Since these flights have different lengths of records, we uniformly sample 10 records for each flight, where only the latitude and longitude are used as features. Therefore, each flight is processed in a 20-length vector for further analysis. Here we select the Chinese flights between Beijing (PEK), Shanghai (PVG), Chengdu (CTU) and Guangzhou (CAN), and US flights between Seattle (SEA), San Francisco (SFO) and Atlanta (ATL) for further analysis. Figure~\ref{fig:f-chinese} \&~\ref{fig:f-us} show the trajectories of Chinese and US flights. By this means, we have the Chinese and US flight trajectory data sets with 85,990 and 33,648 flights, respectively.

Then COR is applied on these two data sets for outlier trajectory detection. Here we set the cluster numbers to be 6 and 3 for these two data sets, and the outlier numbers are both 200. Figure~\ref{fig:o-chinese} \&~\ref{fig:o-us} show the outlier trajectories in these two data sets. There are two kinds of outliers. The first category includes the outliers with extra ranges. Although we focus on 7 airports in China and US, there are some trajectories out of the scope of these airport locations in terms of latitude and longitude. The transmission error and loss lead to that the trajectories of different flights are mixed together. In such cases, the system stores a non-existence trajectory. The second category has the partial trajectories. The flight location is not captured due to the failure of the sensors. These two kinds of outliers detected by COR are advantageous to further analyze the problems in trajectory system, which demonstrates the effectiveness of COR in the real-world application.

\section{Conclusion}
In this paper, we considered the joint clustering and outlier detection problem and proposed the algorithm COR. Different from the existing K-means{-}{-}, we first transformed the original feature space into the partition space according to the relationship between outliers and clusters. Then we provided the objective function based on the Holoentropy, which was partially solved by K-means optimization. Nontrivally, an auxiliary binary matrix was designed so that COR completely solved the challenging problem via K-means{-}{-} on the concatenated binary matrices. Extensive experimental results demonstrated the effectiveness and efficiency of COR significantly over the rivals including K-means{-}{-} and other state-of-the-art outlier detection methods in terms of cluster validity and outlier detection.

% You can push biographies down or up by placing
% a \vfill before or after them. The appropriate
% use of \vfill depends on what kind of text is
% on the last page and whether or not the columns
% are being equalized.

%\vfill

% Can be used to pull up biographies so that the bottom of the last one
% is flush with the other column.
%\enlargethispage{-5in}

\section*{Acknowledgment}
This research is supported in part by the NSF IIS Award 1651902 and U.S. Army Research Office Award W911NF-17-1-0367.

\bibliographystyle{IEEEtran}
\bibliography{egbib}

\begin{IEEEbiography}[{\includegraphics[width=1in,height=1.25in,clip,keepaspectratio]{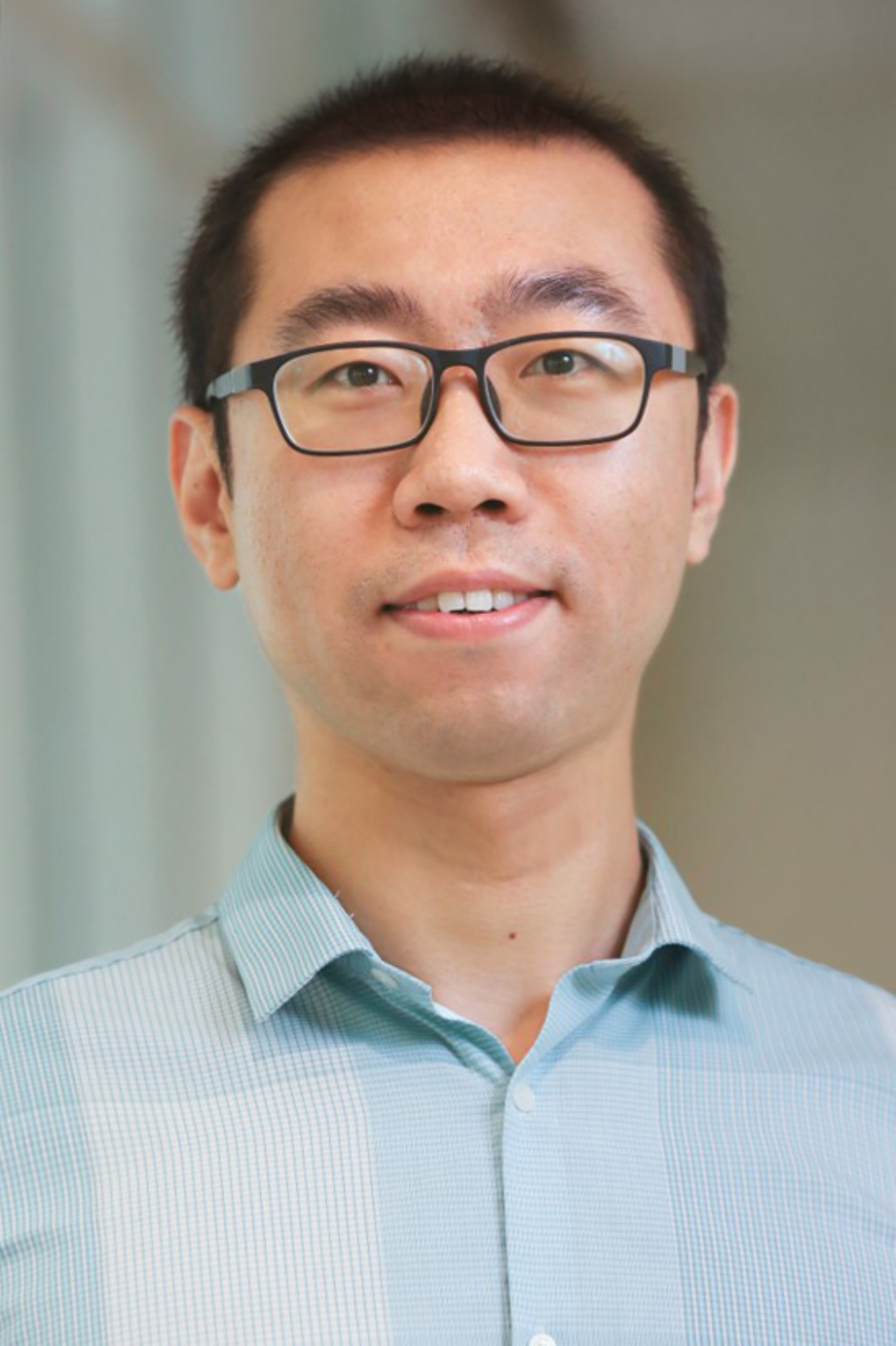}}]{Hongfu Liu}
received his bachelor and master degree in Management Information Systems from the School of Economics and Management, Beihang University, in 2011 and 2014 respectively. He received the Ph.D. degree in computer engineering from Northeastern University, Boston MA, 2018. Currently he is a tenure-track Assistant Professor affiliated with Michtom School of Computer Science at Brandeis University. His research interests generally focus on data mining and machine learning, with special interests in ensemble learning. He has served as the reviewers for many IEEE Transactions journals including TKDE, TNNLS, TIP, and TBD. He has also served on the program committee for the conferences including AAAI, IJCAI, and NIPS. He is the Associate Editor of IEEE Computational Intelligence Magazine.
\end{IEEEbiography}
\vspace{-0.5cm}

\begin{IEEEbiography}[{\includegraphics[width=1in,height=1.25in,clip,keepaspectratio]{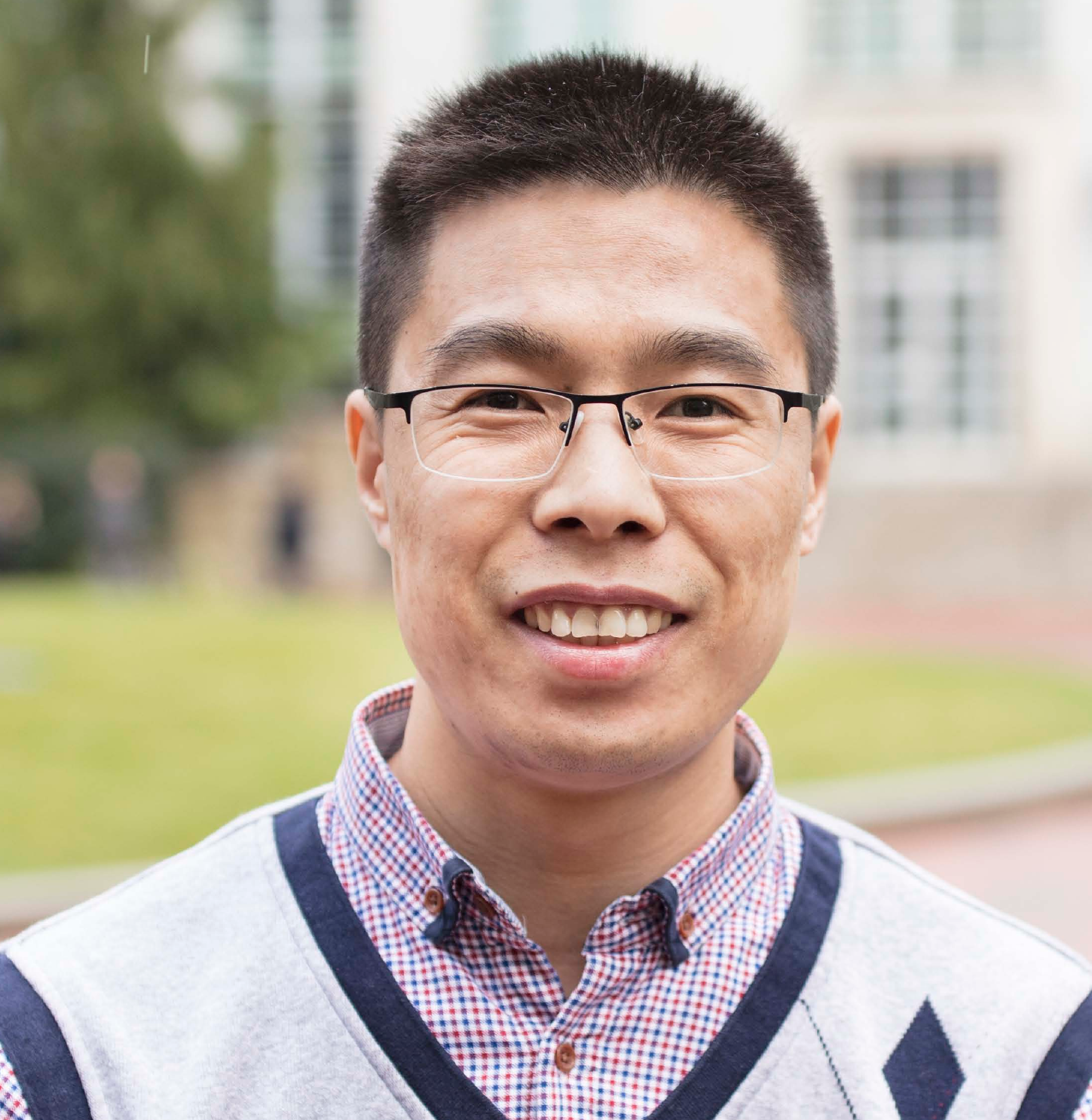}}]{Jun Li} (M'16) received the B.A. in Applied Mathematics from Pan Zhi Hua University in 2006. He received the M.S. in Computer Application from China West Normal University in 2009 and the PhD degree in pattern recognition and intelligence systems from the Nanjing University of Science and Technology in 2015. From Oct. 2012 to July 2013, he was a visiting student at Department of Statistics, Rutgers University, Piscataway, NJ, USA. From Dec. 2015 to Oct. 2018, he was a postdoctoral associate with the Department of Electrical and Computer Engineering, Northeastern University, Boston, MA, USA. He is currently a postdoctoral associate with the Institute of Medical Engineering and Science, Massachusetts Institute of Technology, Cambridge, MA, USA. He is the Associate Editor of IEEE Access. His current research interests include deep learning, reinforcement learning, sparse representations, subspace clustering and recurrent neural networks.
\end{IEEEbiography}
\vspace{-0.5cm}

\begin{IEEEbiography}[{\includegraphics[width=1in,height=1.25in,clip,keepaspectratio]{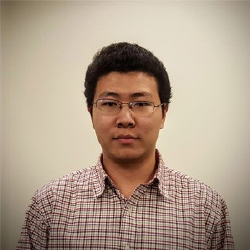}}]{Yue Wu} received the BS and MS degree in Beijing University of Posts and Telecommunications at 2013 and 2016. He is currently a PhD student at Northeastern University. His current research interests are face recognition, object detection and deep learning.
\end{IEEEbiography}
\vspace{-0.5cm}

\begin{IEEEbiography}[{\includegraphics[width=1in,height=1.25in,clip,keepaspectratio]{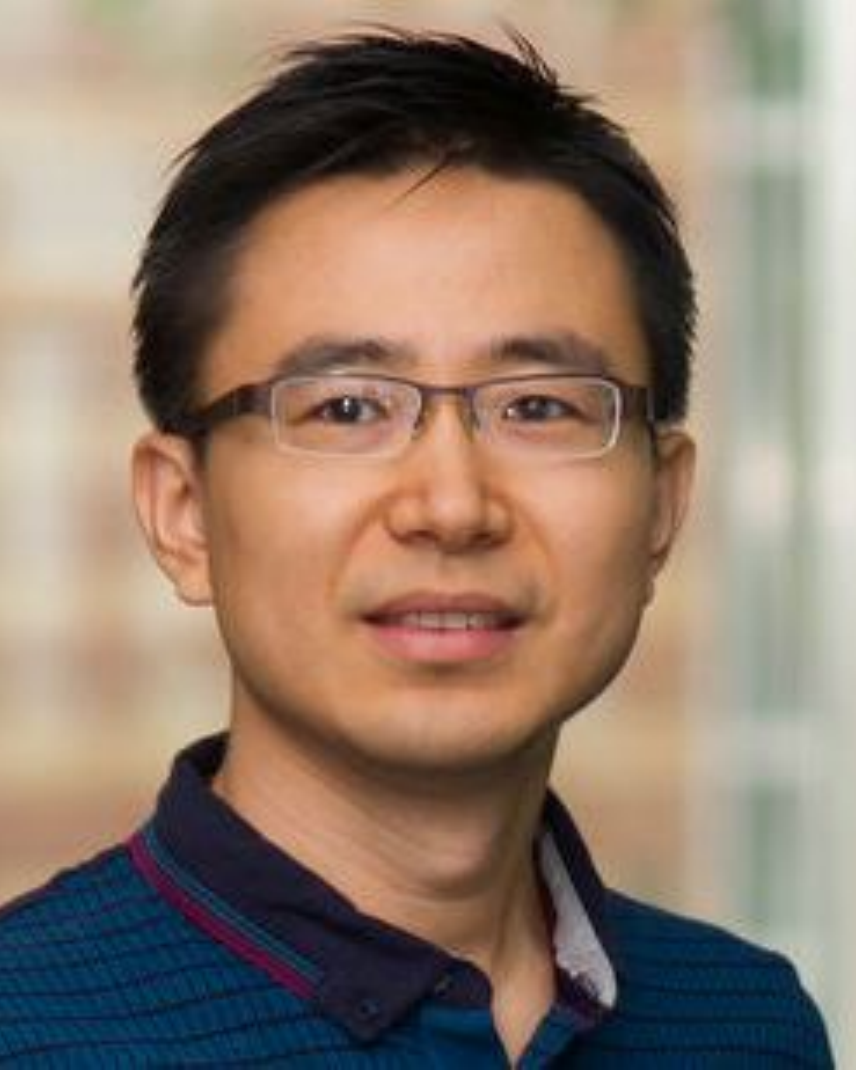}}]{Yun Fu} (S'07-M'08-SM'11-F'19) received the B.Eng. degree in information engineering and the M.Eng. degree in pattern recognition and intelli- gence systems from Xian Jiaotong University, China, respectively, and the M.S. degree in statistics and the Ph.D. degree in electrical and computer engineering from the University of Illinois at Urbana- Champaign, respectively. He is an interdisciplinary faculty member affiliated with College of Engineering and the College of Computer and Information Science at Northeastern University since 2012. His research interests are Machine Learning, Computational Intelligence, Big Data Mining, Computer Vision, Pattern Recognition, and Cyber-Physical Systems. He has extensive publications in leading journals, books/book chapters and international conferences/workshops. He serves as associate editor, chairs, PC member and reviewer of many top journals and international conferences/workshops. He received seven Prestigious Young Investigator Awards from NAE, ONR, ARO, IEEE, INNS, UIUC, Grainger Foundation; nine Best Paper Awards from IEEE, IAPR, SPIE, SIAM; many major Industrial Research Awards from Google, Samsung, and Adobe, etc. He is currently an Associate Editor of the IEEE Transactions on Neural Networks and Leaning Systems (TNNLS). He is fellow of IEEE, IAPR, OSA and SPIE, a Lifetime Distinguished Member of ACM, Lifetime Member of AAAI, and Institute of Mathematical Statistics, member of AAAS, ACM Future of Computing Academy, Global Young Academy (GYA), INNS and Beckman Graduate Fellow during 2007-2008.
\end{IEEEbiography}

% that's all folks
\end{document}